\documentclass[10pt,twocolumn,letterpaper]{article}

\usepackage[margin=1in]{geometry}
\usepackage[T1]{fontenc}
\usepackage[utf8]{inputenc}
\usepackage[hyphens]{url}
\usepackage{graphicx}
\usepackage{natbib}
\bibpunct{(}{)}{;}{a}{,}{,}
\usepackage{caption}
\usepackage{booktabs}
\usepackage{multirow}
\usepackage{pifont}
\usepackage{colortbl}
\usepackage{bm}
\usepackage{amssymb}
\usepackage{amsmath}
\usepackage{amsfonts}
\usepackage{amsthm}
\usepackage{makecell}
\usepackage{algorithm}
\usepackage{algorithmic}
\usepackage[table]{xcolor}
\usepackage{array}
\usepackage{adjustbox}
\usepackage{microtype}
\usepackage[hidelinks]{hyperref}

\newcolumntype{L}[1]{>{\raggedright\arraybackslash}p{#1}}
\newcommand{\cmark}{\ding{51}}
\newcommand{\xmark}{\ding{55}}
\newcommand{\method}{\textsc{COSMO}}

\newcommand{\KL}{\operatorname{KL}}
\newcommand{\softmax}{\operatorname{softmax}}
\newcommand{\simplex}{\Delta^{K-1}}
\newcommand{\ones}{\mathbf 1}
\newcommand{\uniform}{\mathbf u}
\DeclareMathOperator*{\argmax}{arg\,max}

\newtheorem{proposition}{Proposition}[section]
\newtheorem{lemma}[proposition]{Lemma}
\newtheorem{corollary}[proposition]{Corollary}
\theoremstyle{remark}

\title{COSMO: Consensus-Driven Shift Modulation for Source-Free Domain Adaptation}
\author{
Bo Li \quad Junjie Peng \quad Xiaohua Xie \quad Jianhuang Lai\\[2pt]
School of Computer Science and Engineering, Sun Yat-sen University\\
Guangzhou, China\\
\texttt{libo88@mail2.sysu.edu.cn}
}
\date{}

\begin{document}
\maketitle

\begin{abstract}
Source-free domain adaptation (SFDA) adapts a source-trained model to
an unlabeled target domain without source data, a practical setting
under privacy or storage constraints. Yet its self-generated
supervision can reinforce source bias under substantial domain shifts.
Pretrained vision--language models (VLMs) offer complementary semantic
knowledge, but the relative reliability of the source model and VLM
varies across target samples. Existing cross-model guidance does not
explicitly account for this variation and may overwrite valid
source-derived evidence under conflict, a failure we term
\emph{source-derived evidence forgetting}. We formulate VLM-guided
SFDA as a sample-wise reliability-allocation problem and propose
Consensus-Driven Shift Modulation (COSMO). COSMO replaces
expert-to-expert guidance with co-adaptation through an anchored shared
consensus. It first forms a sample-specific initial consensus that
favors the more concentrated prediction. During adaptation, COSMO
re-aggregates both branches' evolving evidence and regulates how far
the resulting consensus moves from its initial anchor based on
consensus uncertainty and training progress. This keeps the shared
supervision anchored yet adaptive. Across four benchmarks, COSMO
achieves state-of-the-art performance under matched VLM backbones.
Further analyses indicate that it better balances the retention of
valid source-derived evidence with the absorption of complementary VLM
evidence. 
\end{abstract}

\section{Introduction}
\label{sec:introduction}

\begin{figure}[t]
\centering
\includegraphics[width=\columnwidth]
{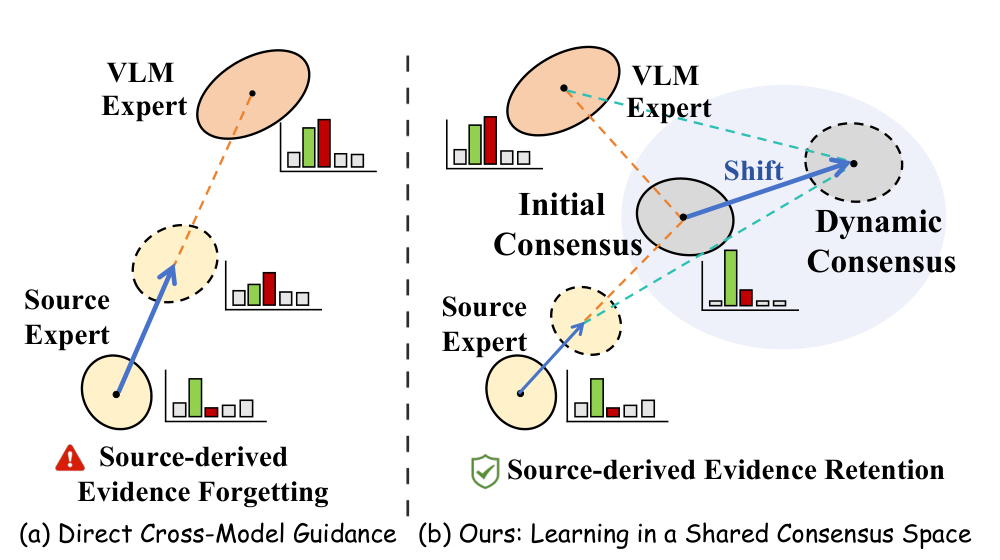}
\caption{
Conceptual comparison of VLM-guided SFDA.
Ellipses summarize target-domain predictions; bars show an
illustrative sample.
(a) Cross-model guidance may overwrite valid source-derived
evidence under expert conflict.
(b) COSMO instead learns from an evolving shared consensus,
retaining source-derived evidence while integrating complementary VLM evidence.
}
\label{fig:teaser}
\end{figure}

Source-free domain adaptation (SFDA) adapts a source-trained model
to an unlabeled target domain without access to source
data~\citep{liang2020shot,xia2021adaptive}. Conventional SFDA
recovers supervision from the source hypothesis and unlabeled target
structure through pseudo-labels, prototypes, neighborhoods, or
prediction regularization. However, under large domain shifts, these
self-generated signals can inherit source bias and progressively
reinforce erroneous decisions.

Pretrained vision-language models (VLMs), such as
CLIP~\citep{radford2021learning}, extend this information boundary
with external visual-semantic knowledge. Recent VLM-guided SFDA
methods customize VLM prompts~\citep{tang2024difo}, correct VLM
predictions~\citep{tang2025prode,peng2026grkv}, or alternate bidirectional transfer
between the task model and VLM~\citep{han2026vsfot}. Despite these
advances, supervision is still exchanged through expert-specific
predictions, with one expert guiding, correcting, or aligning the
other. Such interaction becomes fragile when the experts disagree:
the source model retains task-specific decision evidence learned from
labeled source data, whereas the VLM contributes broader semantics
acquired from image--text pretraining, so either expert may be uniquely
correct on different target samples. Existing cross-model objectives
do not explicitly allocate supervisory influence according to their
sample-wise relative reliability; instead, the dominant signal is
determined indirectly by the objective, loss weights, or update
schedule. VLM guidance can therefore correct a source-biased error,
but may also overwrite an initially correct source-derived decision, as
illustrated in Fig.~\ref{fig:teaser}(a). We call this failure
\emph{source-derived evidence forgetting}. This exposes a central
reliability-allocation problem: how can complementary evidence be
integrated without blindly privileging either expert?

We address this problem with \textbf{Consensus-Driven Shift
Modulation (COSMO)}, which mediates cross-model interaction through a
shared consensus instead of direct expert-to-expert guidance, as shown
in Fig.~\ref{fig:teaser}(b). At initialization, COSMO constructs a
sample-specific consensus from the frozen source expert and the
initially prompted VLM. For each sample, prediction entropy serves as a
label-free proxy for relative predictive concentration and assigns
greater influence to the more concentrated expert; it is not
interpreted as calibrated correctness. Formally, this consensus is an
entropy-conditioned reverse-KL barycenter whose closed-form solution is
a weighted product of experts. By reinforcing compatible class evidence
and mediating conflicts according to relative concentration, it
provides a shared initial anchor without assigning either expert a
fixed teacher role.

During adaptation, the source-initialized target branch and the
prompt-adapted VLM aggregate their current predictions into a dynamic
consensus. Before it is used for supervision, entropy-rank-guided
Consensus Shift Modulation (CSM) rescales its displacement from the
initial anchor according to each sample's rank in the current
consensus-entropy distribution, using a bounded factor that decays
toward one over training. Relatively low-entropy samples remain closer
to the initial anchor, reducing unnecessary drift, whereas
high-entropy samples retain greater plasticity along the direction
already inferred by the dynamic consensus. Both branches then learn
synchronously from the same detached, modulated consensus, and their
updated predictions are re-aggregated to form the next one. This
re-aggregation--modulation--learning loop allows newly
acquired target-domain evidence to enter subsequent supervision while
keeping the initial evidence as a persistent reference.

Our contributions are summarized as follows:
\begin{itemize}
    \item We identify and empirically diagnose
    \emph{source-derived evidence forgetting} in VLM-guided SFDA,
    and formulate the problem as sample-wise reliability allocation
    between heterogeneous experts.

    \item We propose \method{}, which replaces direct cross-model
    guidance with synchronous co-adaptation through an evolving shared
    consensus anchored to its initial state. We formulate the initial
    consensus as an entropy-conditioned reverse-KL barycenter, derive
    its closed-form weighted-PoE solution, and characterize
    decision-retention conditions under expert conflict.

    \item Experiments on four benchmarks demonstrate state-of-the-art
    performance under matched VLM backbones. Ablations and
    conflict-conditioned analyses further show that \method{} better
    retains valid source-derived decisions while absorbing
    complementary VLM evidence.
\end{itemize}

\section{Related Work}

\paragraph{Source-Free Domain Adaptation.}
Without source data, SFDA derives supervision from the source
hypothesis and unlabeled target data. Existing methods reconstruct
source structure from generative models, classifier parameters, or
estimated statistics~\citep{kurmi2021domain,tian2022vdm,
ding2022source}; refine source predictions through information
maximization, pseudo-labeling, or noise-robust
learning~\citep{liang2020shot,chen2022self,yi2023noisy}; or exploit
target-domain geometry and uncertainty~\citep{huang2021hcl,
yang2021nrc,yang2022aad,litrico2023plue,tang2024tpds}.
These signals can nevertheless remain coupled to source-model bias under large domain shifts.

\paragraph{Vision--Language Models for Domain Adaptation.}
With access to source data, VLM-based unsupervised domain adaptation
(UDA) methods learn domain-aware prompts, debias predictions, or align
visual--semantic representations across domains~\citep{
lai2023padclip,singha2023adclip,bai2024pda,du2024damp}.
In SFDA, existing methods customize VLM prompts, distill their
predictions, or iteratively refine cross-model guidance using
adaptation dynamics~\citep{tang2024difo,tang2025prode,
tang2026difoplus,han2026vsfot}. Recent approaches further integrate
heterogeneous visual--semantic evidence through prediction fusion,
neighborhood learning, multimodal alignment, or reliability-aware
curricula~\citep{zhan2024dynamic,tarashima2025vilaad,
peng2025lad,chen2026mmga,chen2026rcl,
lee2025collaborativelearningmultiplefoundation}. While these designs
strengthen external supervision, collaboration is still organized
around directed transfer, fused pseudo-labels, or curriculum-based
selection. \method{} instead lets both evolving branches learn from a
sample-conditioned shared consensus whose evolution is anchored to
their initial evidence.

\paragraph{Knowledge Distillation.}
Recent VLM distillation covers cross-modal affinity, feature, and
alignment transfer~\citep{wu2023tinyclip,yang2024clipkd,
feng2025alignkd}; teacher-augmented multimodal
pretraining~\citep{vasu2024mobileclip,sameni2024sfclip}; adaptation to
unlabeled domains, downstream tasks, and heterogeneous
students~\citep{li2024promptkd,lee2025customkd,jang2025vl2lite}; and
self- or multi-encoder distillation~\citep{wu2024clipself,
kim2025cosmos,cao2025movekd}. Despite increasingly adaptive knowledge
sources, these methods generally preserve a privileged
teacher-to-student direction. \method{} instead uses the same detached
consensus to supervise both branches, without assigning either branch
a fixed teacher role.

\begin{figure*}[t]
    \centering
    \includegraphics[width=\textwidth]{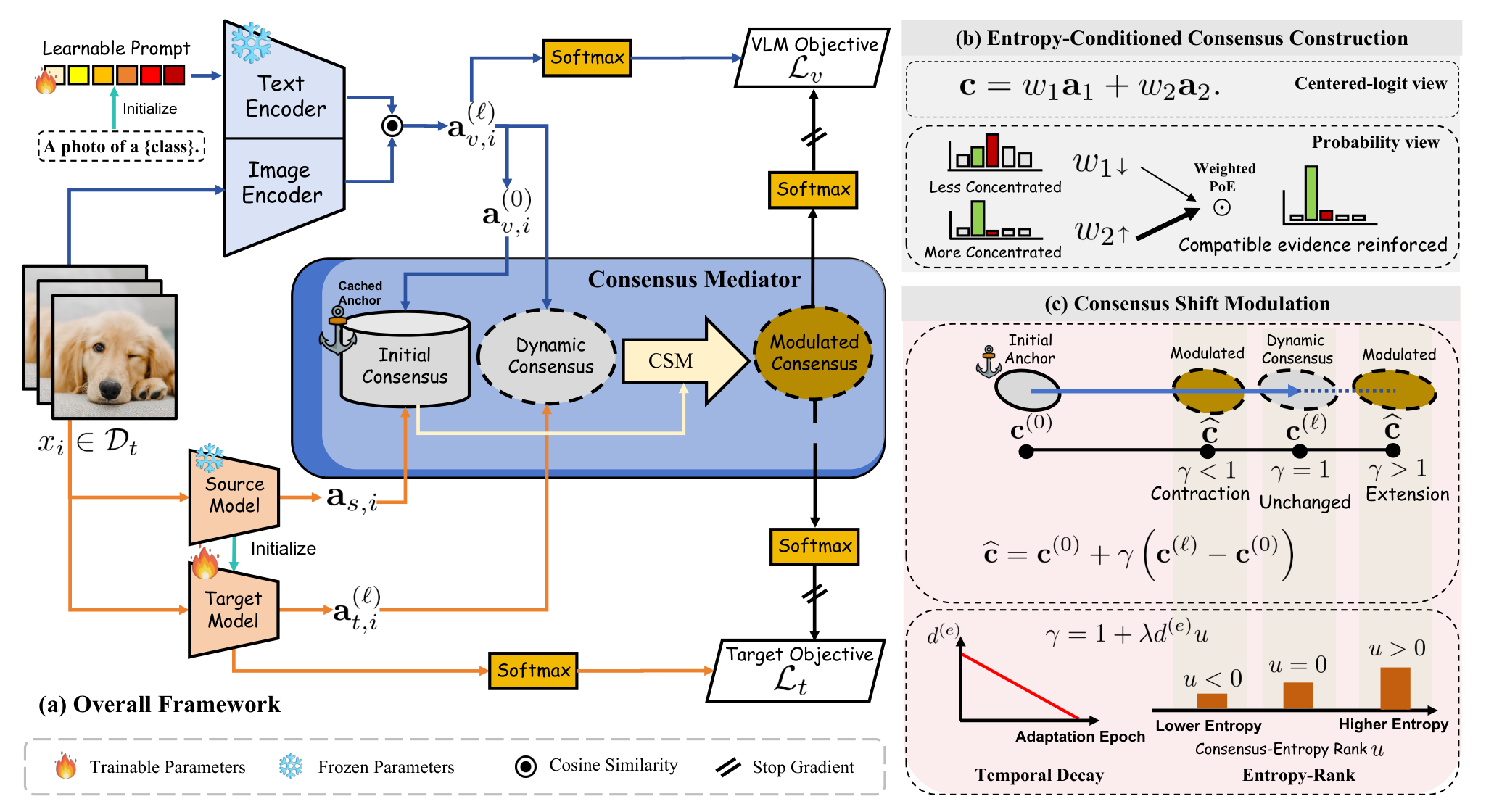}
    \caption{
    Overview of \method{}.
    (a) For each target sample, the frozen source expert and initially
    prompted VLM form a cached consensus anchor, while the current target
    and prompt-adapted VLM branches form a dynamic consensus. CSM modulates
    its anchor-relative displacement, and the detached result supervises
    both branches.
    (b) Entropy-conditioned weighting yields a weighted product of experts
    that emphasizes more concentrated predictions and reinforces
    compatible evidence.
    (c) Consensus-entropy rank contracts, preserves, or extends the
    consensus shift, while temporal decay gradually removes the modulation.
    }
    \label{fig:overview}
\end{figure*}

\section{Method}
\label{sec:method}

\subsection{Problem Setup and Method Overview}
\label{sec:setup}

We consider closed-set source-free domain adaptation with a shared
label space $\mathcal Y=\{1,\ldots,K\}$. Let
$\mathcal D_t=\{x_i\}_{i=1}^{N}$ denote a target set of $N$ unlabeled
samples. We are given a source-trained model $f_s$ without access to
its training data. We retain a frozen copy of $f_s$ as the
\emph{source expert} and initialize the target branch $f_t$ from it.
We additionally use a pretrained vision--language model (VLM) $f_v$.
Its pretrained encoders remain frozen, while its prompt context is
optimized together with $f_t$. We denote the initially prompted VLM
by $f_v^{(0)}$, and the target and VLM branches at optimization step
$\ell$ by $f_t^{(\ell)}$ and $f_v^{(\ell)}$.

For predictor $m\in\{s,t,v\}$ and sample $x_i$, let
$\mathbf z_{m,i}^{(\ell)}\in\mathbb R^K$ be its logits and
$\mathbf p_{m,i}^{(\ell)}
=\softmax(\mathbf z_{m,i}^{(\ell)})$ its prediction. The frozen source
quantities are independent of $\ell$, and we omit their step
superscript when no ambiguity arises. We represent each prediction in
the centered-logit coordinate
\begin{equation}
    \mathbf a_{m,i}^{(\ell)}
    =
    \mathbf\Pi\log\mathbf p_{m,i}^{(\ell)}
    =
    \mathbf\Pi\mathbf z_{m,i}^{(\ell)} .
\label{eq:canonical_prediction}
\end{equation}
Here,
$\mathbf\Pi=\mathbf I-\frac{1}{K}\mathbf1\mathbf1^\top$
is the centering matrix and $\mathbf1$ is the all-ones vector.
Centering removes the additive ambiguity of logits without changing
the corresponding prediction, providing a common coordinate for
consensus construction and anchor-relative displacement. Further
details are given in Appendix~A.1.

Figure~\ref{fig:overview}(a) presents the overall framework.
\method{} addresses three questions: how to allocate the initial
evidence between the source expert and VLM, how to incorporate newly
acquired target-domain evidence into subsequent supervision, and how
far the resulting consensus should move from its initial state. The
first is addressed by the entropy-conditioned consensus construction
in Fig.~\ref{fig:overview}(b), the second by the dynamic
re-aggregation loop in Fig.~\ref{fig:overview}(a), and the third by
Consensus Shift Modulation (CSM) in Fig.~\ref{fig:overview}(c).

\subsection{Entropy-Conditioned Consensus Construction}
\label{sec:initial_consensus}

\paragraph{Sample-wise expert allocation.}
As illustrated in Fig.~\ref{fig:overview}(b), consider an active expert set $\mathcal M$ and an expert
$m\in\mathcal M$. We quantify its
predictive concentration and normalize the resulting scores within
each sample:
\begin{equation}
\begin{aligned}
    r_{m,i}^{(\ell)}
    &=
    \max\!\left\{
        \log K-H\bigl(\mathbf p_{m,i}^{(\ell)}\bigr),\,
        \epsilon
    \right\},\\
    w_{m,i}^{(\ell)}
    &=
    \frac{r_{m,i}^{(\ell)}}
    {\sum_{n\in\mathcal M}r_{n,i}^{(\ell)}} ,
\end{aligned}
\label{eq:allocation_weight}
\end{equation}
where $\epsilon>0$ prevents degenerate weights. Since
$\log K-H(\mathbf p)=\KL(\mathbf p\|\mathbf u)$ for the uniform
distribution $\mathbf u$, $r_{m,i}^{(\ell)}$ measures predictive
concentration relative to uniformity. We use it only to allocate
relative influence between the active experts within a sample, not as
a calibrated estimate of correctness.

\paragraph{Consensus construction.}
Given the allocation weights, we define the shared prediction as the
sample-conditioned reverse-KL barycenter
\begin{equation}
    \mathbf q_{\mathcal M,i}^{(\ell)}
    =
    \arg\min_{\mathbf q\in\Delta^{K-1}}
    \sum_{m\in\mathcal M}
    w_{m,i}^{(\ell)}
    \KL\!\left(
        \mathbf q\|
        \mathbf p_{m,i}^{(\ell)}
    \right).
\label{eq:consensus_barycenter}
\end{equation}
Its closed-form solution can be expressed in centered-logit and
probability form as
\begin{equation}
\begin{aligned}
    \mathbf c_{\mathcal M,i}^{(\ell)}
    &=
    \sum_{m\in\mathcal M}
    w_{m,i}^{(\ell)}
    \mathbf a_{m,i}^{(\ell)},
    \quad
    \mathbf q_{\mathcal M,i}^{(\ell)}
    =
    \softmax\bigl(\mathbf c_{\mathcal M,i}^{(\ell)}\bigr),
    \\[2pt]
    q_{\mathcal M,i,k}^{(\ell)}
    &\propto
    \prod_{m\in\mathcal M}
    \bigl(p_{m,i,k}^{(\ell)}\bigr)^{w_{m,i}^{(\ell)}} .
\end{aligned}
\label{eq:consensus_closed_form}
\end{equation}
Thus, the reverse-KL barycenter corresponds to an entropy-conditioned
product of experts in probability space, pooling compatible class-wise
log evidence while resolving conflicts through sample-specific
allocation. The complete derivation is provided in Appendix~A.2.

At initialization, the active pair is the frozen source expert and the
initially prompted VLM, $\mathcal M=\{s,v\}$. We abbreviate the
resulting centered consensus and distribution as
$\mathbf c_i^{(0)}$ and $\mathbf q_i^{(0)}$, respectively, and cache
$\mathbf c_i^{(0)}$ as the fixed \emph{initial-consensus anchor}.
Under expert conflict, the initial consensus retains the source
decision when, against every competing class, the weighted source
margin exceeds the opposing weighted VLM margin; an analogous
condition holds for the VLM decision. These conditions characterize
decision retention rather than correctness. Their complete forms are
given in Appendix~A.3.

\subsection{Dynamic Consensus Co-Adaptation}
\label{sec:dynamic_consensus}

The initial consensus captures the evidence available before
adaptation. As training proceeds, the target branch acquires
target-domain structure and the VLM prompts become task-specific, so
their predictions and relative contributions should evolve.

At each optimization step $\ell$, we reuse
Eqs.~\eqref{eq:allocation_weight}--\eqref{eq:consensus_closed_form}
with the active pair $\mathcal M=\{t,v\}$. We denote the resulting
centered-logit state and prediction by
$\mathbf c_i^{(\ell)}$ and $\mathbf q_i^{(\ell)}$. They are computed
once from the two pre-update branch predictions. The frozen source
expert does not directly enter this dynamic pair; its initial evidence
persists through the anchor $\mathbf c_i^{(0)}$.

CSM transforms the unmodulated dynamic consensus before the resulting
detached snapshot supervises both branches. Their updated predictions
are re-aggregated at the next step, forming a
re-aggregation--modulation--learning loop that incorporates newly
acquired evidence without direct cross-model chasing.

\subsection{Consensus Shift Modulation}
\label{sec:shift_modulation}

Dynamic re-aggregation incorporates newly acquired evidence, but
accepting each sample's full anchor-relative displacement may
unnecessarily overwrite pre-adaptation evidence recorded by the
initial consensus. We therefore introduce entropy-rank-guided
Consensus Shift Modulation (CSM), illustrated in
Fig.~\ref{fig:overview}(c), to preserve the direction inferred by
dynamic re-aggregation while regulating its magnitude according to
relative consensus uncertainty.

At the beginning of epoch
$e\in\{0,\ldots,E-1\}$, we evaluate the unmodulated dynamic consensus
over the target set and denote the resulting predictions by
$\{\bar{\mathbf q}_i^{(e)}\}_{i=1}^{N}$. We compute
\begin{equation}
\begin{alignedat}{2}
    h_i^{(e)}
    &=
    \frac{H(\bar{\mathbf q}_i^{(e)})}{\log K},
    &\qquad
    u_i^{(e)}
    &=
    2\frac{
        \operatorname{rank}^{(e)}(h_i^{(e)})
    }{N-1}-1,\\
    d^{(e)}
    &=
    1-\frac{e}{E-1},
    &
    \gamma_i^{(e)}
    &=
    1+\lambda d^{(e)}u_i^{(e)}.
\end{alignedat}
\label{eq:entropy_rank_modulation}
\end{equation}
where $\operatorname{rank}^{(e)}(\cdot)\in[0,N-1]$ is the
zero-based average rank among target samples, ordered from lower to
higher consensus entropy. Hence, $u_i^{(e)}\in[-1,1]$ is a centered
uncertainty rank; $E>1$ and $\lambda$ denote the number of epochs and
modulation strength, respectively.

For a step $\ell$ in epoch $e$, CSM rescales the anchor-relative
displacement in centered-logit space:
\begin{equation}
\begin{aligned}
    \Delta\mathbf c_i^{(\ell)}
    &=
    \mathbf c_i^{(\ell)}-\mathbf c_i^{(0)},\\
    \widehat{\mathbf c}_i^{(\ell)}
    &=
    \mathbf c_i^{(0)}
    +\gamma_i^{(e)}
     \Delta\mathbf c_i^{(\ell)},\\
    \widehat{\mathbf q}_i^{(\ell)}
    &=
    \softmax\!\left(
        \widehat{\mathbf c}_i^{(\ell)}
    \right).
\end{aligned}
\label{eq:modulated_consensus}
\end{equation}
Samples with lower consensus-entropy ranks
($u_i^{(e)}<0$) have $\gamma_i^{(e)}<1$ and contract the
anchor-relative displacement, whereas higher-ranked samples
($u_i^{(e)}>0$) have $\gamma_i^{(e)}>1$ and extend it along the
direction inferred by dynamic re-aggregation. When
$u_i^{(e)}=0$ or $d^{(e)}=0$, CSM leaves the dynamic consensus
unchanged. Since $u_i^{(e)}\in[-1,1]$, $d^{(e)}\in[0,1]$, and
$0\leq\lambda<1$, we have
$0<1-\lambda d^{(e)}\leq\gamma_i^{(e)}
\leq1+\lambda d^{(e)}<2$.
Thus, CSM changes only the shift magnitude without canceling or
reversing its direction; the strict bound on $\lambda$ excludes
complete suppression. As $d^{(e)}$ decays to zero, CSM reduces to the
identity transformation. Formal properties are given in
Appendix~A.4.

Finally, the two entropy cues operate at distinct scopes:
$H(\mathbf p_{m,i}^{(\ell)})$ in
Eq.~\eqref{eq:allocation_weight} is compared across experts within
each sample to allocate their relative influence, whereas
$h_i^{(e)}$ is ranked across target samples to regulate the magnitude
of consensus adaptation.

\begin{table*}[t]
\centering

\begingroup
\small
\renewcommand{\arraystretch}{1.08}
\setlength{\tabcolsep}{2.5pt}

\begin{adjustbox}{max width=\linewidth}
\begin{tabular}{@{}l | c | c | c | *{13}{c} | c@{}}
\toprule
\multirow{2}{*}{Method}
& \multirow{2}{*}{Venue}
& \multirow{2}{*}{VLM}
& \multirow{2}{*}{SF}
& \multicolumn{13}{c|}{Office-Home}
& \multicolumn{1}{c}{VisDA-C} \\
\cline{5-17}
\cline{18-18}
& & & 
& A$\to$C
& A$\to$P
& A$\to$R
& C$\to$A
& C$\to$P
& C$\to$R
& P$\to$A
& P$\to$C
& P$\to$R
& R$\to$A
& R$\to$C
& R$\to$P
& Avg.
& S$\to$R \\
\midrule

Source
& --
& --
& --
& 43.6 & 67.0 & 73.8
& 49.9 & 60.2 & 62.8
& 51.6 & 40.8 & 72.5
& 64.4 & 46.4 & 78.2
& 59.3
& 49.0 \\

CLIP
& ICML'21
& B/32
& --
& 64.8 & 86.5 & 87.2
& 77.9 & 86.5 & 87.2
& 77.9 & 64.8 & 87.2
& 77.9 & 64.8 & 86.5
& 79.1
& 87.3 \\

\midrule

SHOT
& ICML'20
& --
& \cmark
& 56.7 & 77.9 & 80.6
& 68.0 & 78.0 & 79.4
& 67.9 & 54.5 & 82.3
& 74.2 & 58.6 & 84.5
& 71.9
& 82.7 \\

NRC
& NeurIPS'21
& --
& \cmark
& 57.7 & 80.3 & 82.0
& 68.1 & 79.8 & 78.6
& 65.3 & 56.4 & 83.0
& 71.0 & 58.6 & 85.6
& 72.2
& 85.9 \\

AaD
& NeurIPS'22
& --
& \cmark
& 59.3 & 79.3 & 82.1
& 68.9 & 79.8 & 79.5
& 67.2 & 57.4 & 83.1
& 72.1 & 58.5 & 85.4
& 72.7
& 88.0 \\

TPDS
& IJCV'24
& --
& \cmark
& 59.3 & 80.3 & 82.1
& 70.6 & 79.4 & 80.9
& 69.8 & 56.8 & 82.1
& 74.5 & 61.2 & 85.3
& 73.5
& 87.6 \\

\midrule

PADCLIP
& ICCV'23
& \multirow{5}{*}{B/16}
& \xmark
& \underline{76.4}
& 90.6
& 90.8
& \underline{86.7}
& 92.3
& 92.0
& 86.0
& 74.5
& 91.5
& \underline{86.9}
& \textbf{79.1}
& 93.1
& 86.7
& \underline{90.9} \\

ADCLIP
& ICCVW'23
&
& \xmark
& 70.9
& 92.5
& \underline{92.1}
& 85.4
& 92.4
& \underline{92.5}
& \underline{86.7}
& 74.3
& \underline{93.0}
& \underline{86.9}
& 72.6
& 93.8
& 86.1
& 90.7 \\

PDA
& AAAI'24
&
& \xmark
& 73.5
& 91.4
& 91.3
& 86.0
& 91.6
& 91.5
& 86.0
& 73.5
& 91.7
& 86.4
& 73.0
& 92.4
& 85.7
& 89.7 \\

DAMP
& CVPR'24
&
& \xmark
& 75.7
& \underline{94.2}
& 92.0
& 86.3
& \underline{94.2}
& 91.9
& 86.2
& \underline{76.3}
& 92.4
& 86.1
& 75.6
& \underline{94.0}
& \underline{87.1}
& \underline{90.9} \\

\rowcolor[gray]{0.92}
\textbf{COSMO}
& --
&
& \cmark
& \textbf{78.8}
& \textbf{94.6}
& \textbf{93.2}
& \textbf{88.0}
& \textbf{94.3}
& \textbf{92.8}
& \textbf{87.6}
& \textbf{78.0}
& \textbf{93.2}
& \textbf{88.5}
& \underline{77.8}
& \textbf{94.4}
& \textbf{88.4}
& \textbf{92.0} \\

\midrule

DIFO
& CVPR'24
& \multirow{4}{*}{B/32}
& \cmark
& 70.6
& 90.6
& 88.8
& 82.5
& 90.6
& 88.8
& 80.9
& 70.1
& 88.9
& 83.4
& 70.5
& 91.2
& 83.1
& 90.3 \\

ProDe
& ICLR'25
&
& \cmark
& 72.7
& \underline{92.3}
& 90.5
& 82.5
& 91.5
& 90.7
& 82.5
& 72.5
& 90.8
& 83.0
& 72.6
& \underline{92.2}
& 84.5
& \underline{91.0} \\

VSFOT
& CVPR'26
&
& \cmark
& \underline{75.0}
& 91.7
& \underline{91.4}
& \underline{83.0}
& \underline{91.8}
& \underline{91.4}
& \underline{83.5}
& \underline{75.0}
& \underline{91.4}
& \underline{83.7}
& \underline{74.3}
& 91.9
& \underline{85.3}
& 90.6 \\

\rowcolor[gray]{0.92}
\textbf{COSMO}
& --
&
& \cmark
& \textbf{76.4}
& \textbf{92.6}
& \textbf{92.0}
& \textbf{85.7}
& \textbf{92.4}
& \textbf{91.9}
& \textbf{84.3}
& \textbf{76.1}
& \textbf{91.7}
& \textbf{85.6}
& \textbf{75.5}
& \textbf{92.7}
& \textbf{86.4}
& \textbf{91.4} \\

\bottomrule
\end{tabular}
\end{adjustbox}
\endgroup

\caption{
Closed-set adaptation results (\%) on Office-Home and VisDA-C.
 ``VLM'' denotes the image-encoder backbone of the vision-language
model, and ``SF'' indicates source-free adaptation.
VLM-guided methods are compared only within matched-backbone groups;
the best and second-best results in each group are shown in \textbf{bold} and \underline{underlined}, respectively.
}
\label{tab:closed_oh_visda}
\end{table*}

\subsection{Objective and Optimization}
\label{sec:optimization}

The modulated consensus is computed once from the pre-update branches
and detached before either branch is optimized. For a mini-batch
$\mathcal B$, let
$\widetilde{\mathbf q}_i
=\operatorname{sg}(\widehat{\mathbf q}_i)$,
where $\operatorname{sg}(\cdot)$ denotes stop-gradient. The branch
objectives are
\begin{equation}
\begin{aligned}
    \mathcal L_v
    &=
    \mathcal L_{\mathrm{IIC}}
    \bigl(\mathbf p_v,\widetilde{\mathbf q}\bigr),\\
    \mathcal L_t
    &=
    \alpha\,
    \mathcal L_{\mathrm{IIC}}
    \bigl(\mathbf p_t,\widetilde{\mathbf q}\bigr)
    +\beta\mathcal L_{\mathrm{CE}}
    -\delta H(\bar{\mathbf p}_t),\\
    \mathcal L_{\method}
    &=
    \mathcal L_t+\mathcal L_v ,
\end{aligned}
\label{eq:training_objective}
\end{equation}
where $\alpha,\beta,\delta\geq0$ weight consensus alignment,
hard-label classification, and prediction diversity, respectively.
Following IIC~\citep{ji2019invariant},
$\mathcal L_{\mathrm{IIC}}$ maximizes batch-level mutual information
between a branch prediction and the shared consensus. For the target
branch, $\mathcal L_{\mathrm{CE}}$ uses the hard consensus decision
$\widehat y_i=\arg\max_k\widehat q_{i,k}$. The batch-mean prediction is
$\bar{\mathbf p}_t
=|\mathcal B|^{-1}\sum_{i\in\mathcal B}\mathbf p_{t,i}$;
minimizing $-H(\bar{\mathbf p}_t)$ encourages prediction diversity and
prevents category collapse.

Both branch losses are formed from the same pre-update consensus
snapshot, although their objectives need not be identical. The
additional target-side terms consolidate the deployed classifier,
whereas the VLM branch adapts only its prompts through soft consensus
alignment. At each step, the target model and VLM prompts are updated
synchronously, while the source expert and pretrained VLM encoders
remain frozen. Only the adapted target branch is retained for
inference. Complete loss definitions and the training algorithm are
provided in Appendices~B.1--B.3.

\section{Experiments}

\subsection{Experimental Setup}
\label{sec:experimental_setup}

\paragraph{Datasets and Evaluation.}
We evaluate \method{} on four standard closed-set SFDA benchmarks:
Office-31~\citep{saenko2010adapting},
Office-Home~\citep{venkateswara2017deep},
VisDA-C~\citep{peng2018visda}, and
DomainNet-126~\citep{saito2019semi}.
We report target-domain top-1 accuracy on Office-31, Office-Home, and
DomainNet-126, and mean per-class accuracy on VisDA-C. Dataset-level
scores average all transfer tasks equally. Unless stated otherwise,
\method{} results are averaged over three independent runs; target
labels are used only for evaluation. Dataset statistics and protocols
are detailed in Appendices~C.1--C.3.

\paragraph{Baselines and Implementation.}
We compare against the source model and zero-shot
CLIP~\citep{radford2021learning}; representative non-VLM SFDA
methods~\citep{liang2020shot,yang2021nrc,yang2022aad,
chen2022contrastive,tang2024tpds}; source-available VLM-guided UDA
methods~\citep{ge2023domain,lai2023padclip,singha2023adclip,
bai2024pda,du2024damp}; and source-free VLM-guided SFDA
methods~\citep{tang2024difo,tang2025prode,han2026vsfot}.
The target branch uses ResNet-50, except ResNet-101 on VisDA-C.
We use CLIP ViT-B/16 or ViT-B/32, freezing its encoders and logit scale
and adapting only the prompt context. VLM-guided comparisons use
matched VLM backbones. We set $\lambda=0.5$ in all main experiments unless otherwise stated. Appendices~D--E detail implementation and
comparison protocols.

\subsection{Comparison with State-of-the-Art Methods}
\label{sec:main_results}

\paragraph{Office-Home and VisDA-C.}
Table~\ref{tab:closed_oh_visda} compares closed-set adaptation methods,
with VLM-guided results grouped by backbone. Using ViT-B/32,
\method{} achieves 86.4\% on Office-Home and 91.4\% on VisDA-C,
outperforming VSFOT and ProDe by 1.1 and 0.4 points, respectively.
It also ranks first on all 12 Office-Home transfers. With ViT-B/16,
\method{} reaches 88.4\% and 92.0\%, exceeding the strongest
source-available, matched-backbone baselines by 1.3 and 1.1 points
despite operating without source data. The gains therefore persist
across both CLIP backbones and benchmarks.

\paragraph{DomainNet-126 and Office-31.}
On DomainNet-126, Table~\ref{tab:closed_domainnet} shows that
\method{} with ViT-B/32 achieves 85.1\%, exceeding ProDe by 1.2
points and ranking first on all 12 transfers in the matched source-free
group. Its ViT-B/16 variant further reaches 88.1\%, outperforming the
strongest source-available, matched-backbone baseline by 0.5 points.
On the smaller, saturated Office-31 benchmark, \method{} obtains
92.7\% with ViT-B/32, matching the best source-free result, and
93.0\% with ViT-B/16. Complete Office-31 results are reported in
Appendix~F.1.

\paragraph{Multiple Foundation Models.}
Under CoMA's matched CLIP and InstructBLIP
configuration~\citep{lee2025collaborativelearningmultiplefoundation},
single-stage \method{} achieves 92.0\% on Office-Home, versus CoMA's
reported 90.7\% with multi-stage distillation. Full results are
provided in Appendix~F.5.

\begin{table*}[t]
\centering

\begingroup
\small
\renewcommand{\arraystretch}{1.08}
\setlength{\tabcolsep}{4pt}

\begin{adjustbox}{max width=\linewidth}
\begin{tabular}{@{}l | c | c | c | *{13}{c}@{}}
\toprule
\multirow{2}{*}{Method}
& \multirow{2}{*}{Venue}
& \multirow{2}{*}{VLM}
& \multirow{2}{*}{SF}
& \multicolumn{13}{c}{DomainNet-126} \\
\cline{5-17}
& & & 
& C$\to$P
& C$\to$R
& C$\to$S
& P$\to$C
& P$\to$R
& P$\to$S
& R$\to$C
& R$\to$P
& R$\to$S
& S$\to$C
& S$\to$P
& S$\to$R
& Avg. \\
\midrule

Source
& --
& --
& --
& 44.6 & 59.8 & 47.5
& 53.3 & 75.3 & 46.2
& 55.3 & 62.7 & 46.4
& 55.1 & 50.7 & 59.5
& 54.7 \\

CLIP
& ICML'21
& B/32
& --
& 77.6 & 89.9 & 74.5
& 78.9 & 89.9 & 74.5
& 78.9 & 77.6 & 74.5
& 78.9 & 77.6 & 89.9
& 80.2 \\

\midrule

SHOT
& ICML'20
& --
& \cmark
& 63.5 & 78.2 & 59.5
& 67.9 & 81.3 & 61.7
& 67.7 & 67.6 & 57.8
& 70.2 & 64.0 & 78.0
& 68.1 \\

NRC
& NeurIPS'21
& --
& \cmark
& 62.6 & 77.1 & 58.3
& 62.9 & 81.3 & 60.7
& 64.7 & 69.4 & 58.7
& 69.4 & 65.8 & 78.7
& 67.5 \\

AdaCon
& CVPR'22
& --
& \cmark
& 60.8 & 74.8 & 55.9
& 62.2 & 78.3 & 58.2
& 63.1 & 68.1 & 55.6
& 67.1 & 66.0 & 75.4
& 65.4 \\

TPDS
& IJCV'24
& --
& \cmark
& 62.9 & 77.1 & 59.8
& 65.6 & 79.0 & 61.5
& 66.4 & 67.0 & 58.2
& 68.6 & 64.3 & 75.3
& 67.1 \\

\midrule

DAPL
& TNNLS'23
& \multirow{4}{*}{B/16}
& \xmark
& 83.3 & 92.4 & 81.1
& 86.4 & 92.1 & 81.0
& 86.7 & 83.3 & 80.8
& 86.8 & 83.5 & 91.9
& 85.8 \\

ADCLIP
& ICCVW'23
&
& \xmark
& 84.3 & \underline{93.7} & 82.4
& \underline{87.5} & \underline{93.5} & 82.4
& \underline{87.3} & 84.5 & 81.6
& \underline{87.9} & 84.8 & 93.0
& 86.9 \\

DAMP
& CVPR'24
&
& \xmark
& \underline{86.4} & 93.3 & \underline{83.5}
& 87.2 & 93.4 & \textbf{84.1}
& 87.2 & \underline{86.5} & \underline{82.5}
& 87.3 & \textbf{86.6} & \underline{93.4}
& \underline{87.6} \\

\rowcolor[gray]{0.92}
\textbf{COSMO}
& --
&
& \cmark
& \textbf{86.5} & \textbf{94.0} & \textbf{83.8}
& \textbf{88.0} & \textbf{93.8} & \textbf{84.1}
& \textbf{87.7} & \textbf{86.8} & \textbf{83.8}
& \textbf{88.4} & \textbf{86.6} & \textbf{94.0}
& \textbf{88.1} \\

\midrule

DIFO
& CVPR'24
& \multirow{4}{*}{B/32}
& \cmark
& 76.6 & 87.2 & 74.9
& 80.0 & 87.4 & 75.6
& 80.8 & 77.3 & 75.5
& 80.5 & 76.7 & 87.3
& 80.0 \\

ProDe
& ICLR'25
&
& \cmark
& \underline{81.4} & \underline{92.2} & \underline{77.8}
& \underline{83.3} & \underline{91.9} & \underline{78.6}
& \underline{84.9} & \underline{81.3} & \underline{78.9}
& \underline{83.8} & \underline{81.3} & \underline{91.8}
& \underline{83.9} \\

VSFOT
& CVPR'26
&
& \cmark
& 79.5 & 90.4 & 77.2
& 83.0 & 90.5 & 77.5
& 83.3 & 79.8 & 76.9
& 83.5 & 79.6 & 90.4
& 82.6 \\

\rowcolor[gray]{0.92}
\textbf{COSMO}
& --
&
& \cmark
& \textbf{82.8} & \textbf{92.7} & \textbf{79.7}
& \textbf{84.8} & \textbf{92.7} & \textbf{80.2}
& \textbf{85.4} & \textbf{82.8} & \textbf{79.8}
& \textbf{84.7} & \textbf{82.9} & \textbf{92.7}
& \textbf{85.1} \\

\bottomrule
\end{tabular}
\end{adjustbox}
\endgroup

\caption{
Closed-set adaptation results (\%) on DomainNet-126.
The notation and matched-backbone comparison conventions follow
Table~\ref{tab:closed_oh_visda}.
}
\label{tab:closed_domainnet}
\end{table*}

\begin{table}[t]
\centering
\footnotesize
\renewcommand{\arraystretch}{1.10}
\setlength{\tabcolsep}{2.5pt}

\begin{adjustbox}{max width=\linewidth}
\begin{tabular}{@{}lccccc@{}}
\toprule
\multirow{2}{*}{Setting}
& \multicolumn{2}{c}{Re-aggregation}
& \multirow{2}{*}{CSM}
& \multirow{2}{*}{\makecell{Office-\\Home}}
& \multirow{2}{*}{\makecell{DomainNet-\\126}} \\
\cmidrule(lr){2-3}
& Target & VLM & & & \\
\midrule

VLM $\rightarrow$ Target
& \multicolumn{2}{c}{--}
& -- & 81.2 & 80.3 \\

VLM $\leftrightarrow$ Target
& \multicolumn{2}{c}{--}
& -- & 84.8 & 82.0 \\

\midrule

\multirow{4}{*}{\makecell[l]{Shared\\consensus}}
& \xmark & \xmark & \xmark
& 82.6 & 80.8 \\

& \cmark & \xmark & \xmark
& 85.4 & 83.4 \\

& \xmark & \cmark & \xmark
& 85.1 & 82.9 \\

& \cmark & \cmark & \xmark
& 85.8 & 84.6 \\

\midrule

\rowcolor[gray]{0.92}
\textbf{\method{}}
& \cmark & \cmark & \cmark
& \textbf{86.4} & \textbf{85.1} \\

\bottomrule
\end{tabular}
\end{adjustbox}

\caption{
Ablation of supervision topology, dynamic consensus re-aggregation,
and CSM using CLIP ViT-B/32 (12-shift average accuracies, \%).
For shared-consensus variants, checkmarks and crosses under Target/VLM
denote current and initial branch predictions, respectively.
}
\label{tab:component_ablation}
\end{table}

\subsection{Ablation Study}
\label{sec:ablation}

\paragraph{Component Contributions.}
Table~\ref{tab:component_ablation} separates supervision topology,
dynamic consensus re-aggregation, and CSM. Reciprocal cross-model
transfer improves one-way VLM distillation by 3.6/1.7 points on
Office-Home/DomainNet-126. A fixed shared consensus obtains
82.6/80.8\%, whereas re-aggregating it with the current target or VLM
prediction raises accuracy to 85.4/83.4\% and 85.1/82.9\%,
respectively. Using both evolving branches performs best without CSM
at 85.8/84.6\%, exceeding reciprocal transfer by 1.0/2.6 points.
This indicates that the gain arises from joint dynamic re-aggregation,
not merely from replacing directed transfer with a fixed shared target.
CSM further adds 0.6/0.5 points, yielding the full results of
86.4/85.1\%. Additional component configurations and objective
ablations are reported in Appendices~G.1 and G.3,
respectively.

\begin{table}[t]
\centering
\small
\renewcommand{\arraystretch}{1.10}
\setlength{\tabcolsep}{2.6pt}

\begin{adjustbox}{max width=\linewidth}
\begin{tabular}{@{}llcccc@{}}
\toprule
\multirow{2}{*}{Fusion}
& \multirow{2}{*}{Weighting}
& \multicolumn{2}{c}{Office-Home}
& \multicolumn{2}{c}{VisDA-C} \\
\cmidrule(lr){3-4}
\cmidrule(lr){5-6}
& & Init. & Adapted
  & Init. & Adapted \\
\midrule

\multirow{2}{*}{MoE}
& Uniform
& 80.7 & 85.5
& 84.1 & 89.8 \\
& Expert entropy
& 81.1 & 85.9
& 84.4 & 90.4 \\

\midrule

\multirow{5}{*}{PoE}
& Uniform
& 81.2 & 85.7
& 83.9 & 90.9 \\
& MSP
& \textbf{81.6} & 86.2
& 84.3 & 91.3 \\
& Margin
& 81.0 & 86.1
& 84.0 & \textbf{91.4} \\
& Expert entropy (norm.)
& 77.5 & 84.6
& 80.4 & 90.4 \\
& \textbf{Expert entropy}
& \textbf{81.6} & \textbf{86.4}
& \textbf{84.5} & \textbf{91.4} \\

\bottomrule
\end{tabular}
\end{adjustbox}

\caption{
Ablation of consensus construction using CLIP ViT-B/32.
``Init.''/``Adapted'' denote initial-consensus/adapted-target
accuracies (\%); Office-Home is averaged over 12 shifts.
MoE/PoE denote weighted arithmetic/geometric pooling, MSP denotes
maximum softmax probability, and ``norm.'' applies unsupervised global
logit-scale normalization before expert-entropy computation.
Best values are bold.
}
\label{tab:consensus_construction}
\end{table}

\paragraph{Consensus Construction.}
Table~\ref{tab:consensus_construction} separates pooling geometry from
the allocation cue. Under MoE, expert-entropy weighting improves
adapted accuracy over uniform weighting by 0.4/0.6 points on
Office-Home/VisDA-C. PoE further improves adapted accuracy under both
uniform weighting (0.2/1.1 points) and expert-entropy weighting
(0.5/1.0 points). Within PoE, expert entropy achieves the best or
tied-best result in all four evaluations, reaching 81.6/86.4\% on
Office-Home and 84.5/91.4\% on VisDA-C for initial/adapted accuracy.
MSP and margin remain competitive but are not consistently optimal,
while logit-scale normalization degrades every result, particularly
at initialization. We therefore use expert-entropy-weighted PoE
throughout.

\subsection{Mechanism Analysis}
\label{sec:mechanism}

\begin{figure}[t]
\centering
\includegraphics[width=\columnwidth]{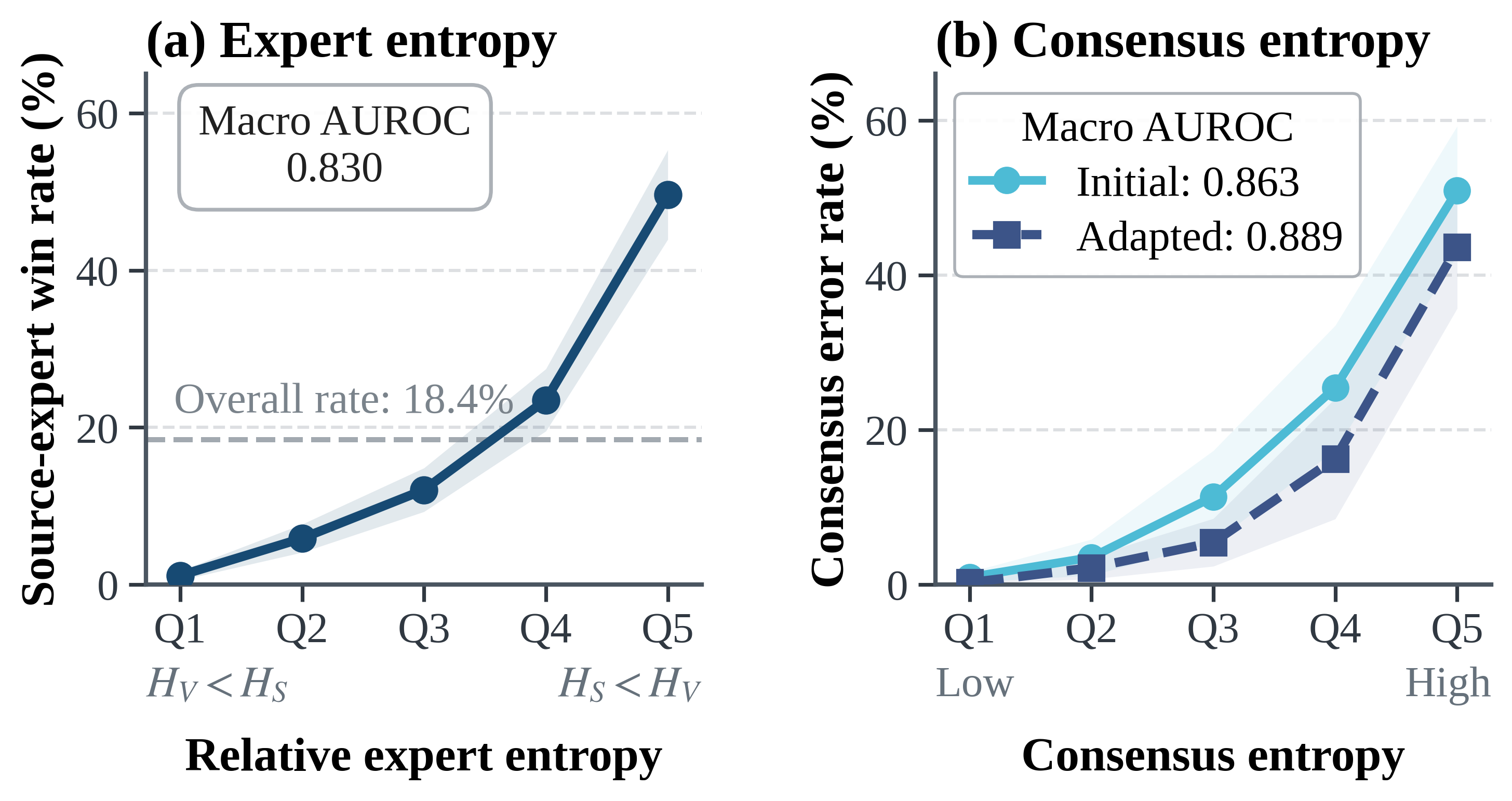}
\caption{
Informativeness of the two entropy cues on Office-Home.
(a) Source-expert win rate among samples where exactly one initial
expert is correct, binned by
$\Delta H=H_V-H_S$; the dashed line denotes the overall win rate.
(b) Consensus error rate across consensus-entropy quintiles before and
after adaptation.
Curves show 12-shift means with 95\% confidence intervals; boxed
values report macro AUROC. Labels are used only for post-hoc analysis.
}
\label{fig:entropy_informativeness}
\end{figure}

\paragraph{Informativeness of Entropy Cues.}
Figure~\ref{fig:entropy_informativeness} evaluates the two entropy
orderings used by \method{}. When only one initial expert is correct,
the source-expert win rate rises monotonically as its entropy decreases
relative to the VLM, yielding a macro AUROC of 0.830. Consensus error
likewise increases across consensus-entropy quintiles, with macro
AUROCs of 0.863 before adaptation and 0.889 afterward. Thus, relative
expert entropy guides within-sample allocation, while consensus entropy
provides the cross-sample ordering used by CSM.

\begin{figure}[!t]
\centering
\includegraphics[width=0.9\columnwidth]
{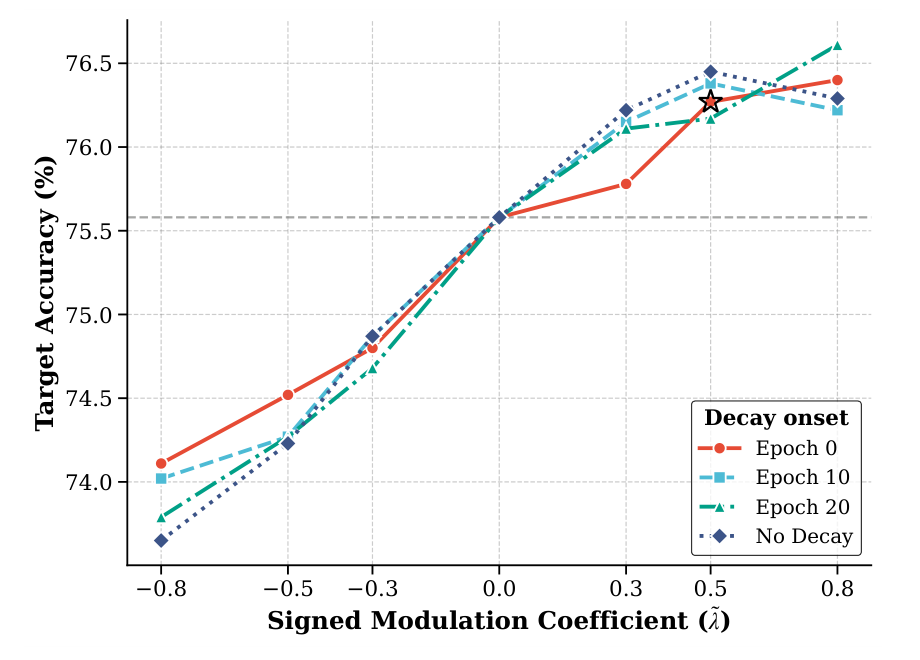}
\caption{
Directional validation of CSM on Office-Home A$\to$C.
Positive $\tilde{\lambda}$ follows the proposed entropy-rank
assignment, negative values reverse it, and zero disables modulation.
Curves compare different decay schedules; the star marks the default
setting.
}
\label{fig:csm_direction}
\end{figure}

\paragraph{Directional Validation of CSM.}
Figure~\ref{fig:csm_direction} tests the direction and temporal schedule
of CSM using a signed diagnostic coefficient.
Across decay schedules, all tested positive settings outperform the
unmodulated baseline, whereas reversing the entropy-rank assignment
consistently reduces accuracy. The gain therefore depends on
contracting shifts for relatively low-entropy samples and extending
them for high-entropy samples, rather than on generic displacement
rescaling. Performance remains stable over the tested positive range
and decay schedules; we use $\lambda=0.5$ with decay starting at
epoch~$0$.

\paragraph{Conflict-Conditioned Adaptation Outcomes.}
Figure~\ref{fig:conflict_outcomes} evaluates retention and absorption
under initial expert conflict.
Direct VLM distillation exhibits a pronounced imbalance, retaining
only $42.8\%$ of source-only correct decisions while absorbing
$95.2\%$ of VLM-only correct decisions. \method{} improves both
criteria over DIFO and ProDe, reaching $78.1\%$ source-only retention
and $91.1\%$ VLM-only absorption. Without CSM, these values decrease
to $76.0\%$ and $90.0\%$, respectively. The results indicate that
shared-consensus adaptation improves the retention--absorption balance,
with CSM providing a further gain on both conflict subsets.

\begin{figure}[!t]
\centering
\includegraphics[width=0.9\columnwidth]
{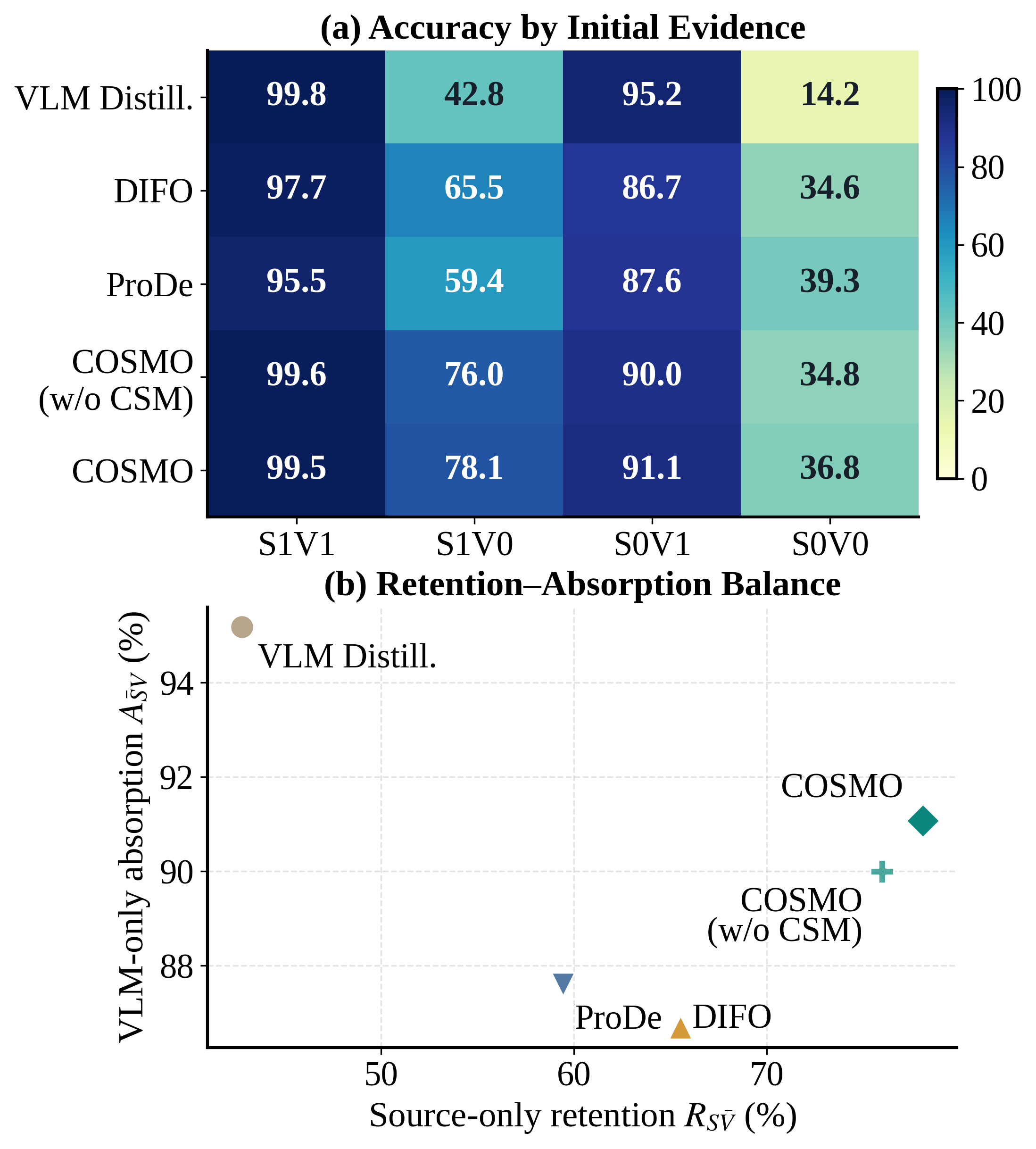}
\caption{
Conflict-conditioned outcomes on Office-Home.
Samples are grouped by the initial correctness of the source expert
($S$) and VLM ($V$).
(a) Accuracy on the four subsets.
(b) Source-only retention $R_{S\bar V}$ versus VLM-only absorption
$A_{\bar S V}$.
Results are averaged over 12 shifts; labels are used only for
post-hoc analysis.
}
\label{fig:conflict_outcomes}
\end{figure}

\section{Conclusion}

We studied VLM-guided SFDA as a sample-wise reliability-allocation
problem and identified source-derived evidence forgetting under expert
conflict. We proposed \method{}, which replaces direct cross-model
guidance with co-adaptation through an anchored shared consensus. An
entropy-conditioned initial consensus allocates expert influence,
dynamic re-aggregation incorporates evolving target evidence, and CSM
regulates departure from the initial anchor. Across four benchmarks,
\method{} achieves state-of-the-art performance under matched VLM
backbones. Ablations and mechanism analyses support the proposed
reliability cues and design choices, while conflict-conditioned results
show a better balance between retaining valid source-derived decisions
and absorbing complementary VLM evidence. These results demonstrate
anchored yet adaptive consensus supervision as an effective alternative
to direct cross-model transfer in VLM-guided SFDA.

\clearpage
\section*{Appendix Overview}

These appendices complement the main text with complete
mathematical, experimental, and reproducibility details. Section~A
provides the centered-logit derivations, the entropy-conditioned
reverse-KL barycenter, decision-retention conditions, and formal
properties of Consensus Shift Modulation (CSM). Section~B specifies the
complete objective and optimization algorithm, while Secs.~C--E describe
the datasets, evaluation, implementation, and comparison protocols.
Section~F reports complete benchmark results and extensions; Sec.~G
provides additional component, consensus-construction, and objective
ablations; Sec.~H presents extended representation and sensitivity
analyses; and Sec.~I evaluates computational efficiency.

\appendix

\renewcommand{\theequation}{S\arabic{equation}}
\renewcommand{\thefigure}{S\arabic{figure}}
\renewcommand{\thetable}{S\arabic{table}}
\setcounter{equation}{0}
\setcounter{figure}{0}
\setcounter{table}{0}

\section{Additional Mathematical Derivations and Properties}
\label{app:mathematical_properties}

\subsection{Centered-Logit Representation}
\label{app:centered_logits}

We first make explicit the representation used in Eq.~(1) of the main
paper. Let $K\geq 2$, let $\ones\in\mathbb R^K$ be the all-ones vector,
and define
\begin{equation}
    \mathbf\Pi
    =
    \mathbf I-\frac{1}{K}\ones\ones^\top .
    \label{eq:centering_matrix}
\end{equation}
The matrix $\mathbf\Pi$ is the orthogonal projector onto
\(\ones^\perp=\{\mathbf a\in\mathbb R^K:\ones^\top\mathbf a=0\}\):
\begin{equation}
    \mathbf\Pi^\top=\mathbf\Pi,\qquad
    \mathbf\Pi^2=\mathbf\Pi,\qquad
    \mathbf\Pi\ones=\mathbf 0 .
    \label{eq:projection_properties}
\end{equation}

\begin{proposition}[Softmax-invariant centered coordinates]
\label{prop:centered_coordinates}
Let $\mathbf z\in\mathbb R^K$ be finite logits and
\(\mathbf p=\operatorname{softmax}(\mathbf z)\). Then
\begin{equation}
    \mathbf a
    :=
    \mathbf\Pi\log\mathbf p
    =
    \mathbf\Pi\mathbf z,
    \qquad
    \operatorname{softmax}(\mathbf a)=\mathbf p .
    \label{eq:centered_logit_identity}
\end{equation}
Moreover, $\mathbf a$ is the unique zero-mean logit vector whose softmax is
$\mathbf p$.
\end{proposition}

\begin{proof}
Writing \(Z(\mathbf z)=\sum_{j=1}^K\exp(z_j)\), we have
\begin{equation}
    \log\mathbf p
    =
    \mathbf z-\log Z(\mathbf z)\ones .
    \label{eq:log_softmax_vector}
\end{equation}
Multiplication by $\mathbf\Pi$ removes the class-independent term, which
proves \(\mathbf\Pi\log\mathbf p=\mathbf\Pi\mathbf z\).
Furthermore,
\[
    \mathbf\Pi\mathbf z
    =
    \mathbf z-\bar z\ones,
    \qquad
    \bar z=\frac{1}{K}\ones^\top\mathbf z .
\]
Softmax is invariant to addition of a common scalar, hence
\(\operatorname{softmax}(\mathbf\Pi\mathbf z)
=\operatorname{softmax}(\mathbf z)=\mathbf p\).

For uniqueness, suppose that $\mathbf b\in\ones^\perp$ and
\(\operatorname{softmax}(\mathbf b)=\mathbf p\). Equality of two softmax
outputs implies \(\mathbf b=\log\mathbf p+c\ones\) for some
\(c\in\mathbb R\). Applying $\mathbf\Pi$ and using
\(\mathbf\Pi\mathbf b=\mathbf b\) gives
\(\mathbf b=\mathbf\Pi\log\mathbf p=\mathbf a\).
\end{proof}

The representation preserves all pairwise log odds:
\begin{equation}
    a_j-a_k
    =
    z_j-z_k
    =
    \log\frac{p_j}{p_k}
    \quad\text{for every }j,k.
    \label{eq:pairwise_log_odds}
\end{equation}
It therefore preserves both the distribution and its top-1 decision while
removing the arbitrary additive logit offset. In particular, for any
\(c\in\mathbb R\),
\(\operatorname{softmax}(\mathbf z+c\ones)
=\operatorname{softmax}(\mathbf z)\), and every member of this equivalence
class maps to the same $\mathbf a$.

\paragraph{What centering does not remove.}
Centering is a coordinate choice, not confidence calibration. For
\(\tau>0\),
\(\operatorname{softmax}(\tau\mathbf a)\) generally differs from
\(\operatorname{softmax}(\mathbf a)\); thus, multiplicative logit-scale
differences between heterogeneous experts remain. Correspondingly, the
entropy-conditioned allocation in the main paper uses the raw predictive
concentration of each expert and should not be interpreted as calibrated
correctness.

\paragraph{Concentration relative to uniformity.}
Let \(\uniform=(1/K)\ones\). For any strictly positive categorical
distribution $\mathbf p$,
\begin{equation}
    \log K-H(\mathbf p)
    =
    \sum_{k=1}^K p_k\log\frac{p_k}{1/K}
    =
    \operatorname{KL}(\mathbf p\Vert\uniform)\geq 0 .
    \label{eq:concentration_as_kl}
\end{equation}
Equality holds only for the uniform distribution. This identity explains
why Eq.~(2) of the main paper measures departure from uniformity. It does
not imply that a more concentrated prediction is necessarily more
accurate under domain shift.

\subsection{Derivation of the Entropy-Conditioned Reverse-KL Barycenter}
\label{app:reverse_kl_barycenter}

Fix a target sample $x_i$ and an optimization step $\ell$. Let
\(\mathcal M\) be the active expert set and, for each \(m\in\mathcal M\),
let
\[
    \mathbf p_{m,i}^{(\ell)}
    =
    \operatorname{softmax}\!\left(\mathbf z_{m,i}^{(\ell)}\right)
    \in\operatorname{int}(\simplex).
\]
Finite softmax logits make every class probability strictly positive. As
in Eq.~(2) of the main paper, define
\begin{equation}
\begin{split}
    r_{m,i}^{(\ell)}
    &=
    \max\!\left\{
        \log K-H\!\left(\mathbf p_{m,i}^{(\ell)}\right),
        \epsilon
    \right\},\\
    w_{m,i}^{(\ell)}
    &=
    \frac{r_{m,i}^{(\ell)}}
    {\sum_{n\in\mathcal M}r_{n,i}^{(\ell)}},
    \qquad \epsilon>0 .
    \label{eq:entropy_conditioned_weights}
\end{split}
\end{equation}
Consequently,
\begin{equation}
    w_{m,i}^{(\ell)}>0,
    \qquad
    \sum_{m\in\mathcal M}w_{m,i}^{(\ell)}=1.
    \label{eq:weight_simplex}
\end{equation}
The weights depend on the current expert predictions, but they are fixed
with respect to the barycenter variable $\mathbf q$ in the following
optimization.

To simplify notation within the derivation, omit $i$ and $\ell$, and write
\(\mathbf p_m=\mathbf p_{m,i}^{(\ell)}\) and
\(w_m=w_{m,i}^{(\ell)}\). The shared prediction in Eq.~(3) of the main
paper is
\begin{equation}
    \mathbf q_{\mathcal M}^{\star}
    =
    \arg\min_{\mathbf q\in\simplex}
    \sum_{m\in\mathcal M}
    w_m\operatorname{KL}(\mathbf q\Vert\mathbf p_m).
    \label{eq:reverse_kl_problem}
\end{equation}

\begin{proposition}[Closed-form reverse-KL barycenter]
\label{prop:reverse_kl_solution}
The problem in Eq.~\eqref{eq:reverse_kl_problem} has the unique solution
\begin{equation}
    q_{\mathcal M,k}^{\star}
    =
    \frac{\displaystyle\prod_{m\in\mathcal M}p_{m,k}^{\,w_m}}
         {\displaystyle\sum_{j=1}^K
          \prod_{m\in\mathcal M}p_{m,j}^{\,w_m}},
    \qquad k=1,\ldots,K .
    \label{eq:weighted_poe_solution}
\end{equation}
\end{proposition}

\begin{proof}
Introduce a Lagrange multiplier $\eta$ for
\(\sum_kq_k=1\). On the simplex interior, the Lagrangian is
\begin{equation}
\begin{aligned}
    \mathcal L(\mathbf q,\eta)
    &=
    \sum_{m\in\mathcal M}w_m
    \sum_{k=1}^K q_k\log\frac{q_k}{p_{m,k}} \\
    &\quad
    +\eta\left(\sum_{k=1}^Kq_k-1\right).
\end{aligned}
\label{eq:barycenter_lagrangian}
\end{equation}
Using \(\sum_mw_m=1\), stationarity with respect to $q_k$ gives
\begin{equation}
    \frac{\partial\mathcal L}{\partial q_k}
    =
    \log q_k+1-\sum_{m\in\mathcal M}w_m\log p_{m,k}+\eta
    =
    0 .
    \label{eq:barycenter_stationarity}
\end{equation}
Therefore,
\begin{equation}
    q_k
    =
    C\exp\!\left(
        \sum_{m\in\mathcal M}w_m\log p_{m,k}
    \right)
    =
    C\prod_{m\in\mathcal M}p_{m,k}^{\,w_m},
    \label{eq:unnormalized_poe}
\end{equation}
where the class-independent constant $C>0$ is determined by
normalization. This yields Eq.~\eqref{eq:weighted_poe_solution}.

For completeness, define
\[
    g_k=\prod_{m\in\mathcal M}p_{m,k}^{\,w_m},
    \qquad
    G=\sum_{j=1}^K g_j,
    \qquad
    q_k^\star=g_k/G .
\]
Using the standard convention \(0\log 0=0\), equivalently
\(0\log(0/a)=0\) for \(a>0\), we have, for every
\(\mathbf q\in\simplex\),
\begin{equation}
\begin{split}
    \sum_{m\in\mathcal M}w_m
    \operatorname{KL}(\mathbf q\Vert\mathbf p_m)
    &=
    \sum_{k=1}^Kq_k\log\frac{q_k}{g_k}\\
    &=
    \operatorname{KL}(\mathbf q\Vert\mathbf q^\star)-\log G .
    \label{eq:barycenter_kl_decomposition}
\end{split}
\end{equation}
Because all $g_k$ are positive, $\mathbf q^\star$ lies in the simplex
interior. The nonnegativity and equality condition of KL divergence show
that it is the unique global minimizer, including over the simplex
boundary.
\end{proof}

\begin{corollary}[Centered-logit and weighted-PoE equivalence]
\label{cor:centered_logit_poe}
Let
\begin{equation}
    \mathbf a_m
    =
    \mathbf\Pi\log\mathbf p_m
    =
    \mathbf\Pi\mathbf z_m .
    \label{eq:expert_centered_coordinates}
\end{equation}
Then the barycenter can equivalently be written as
\begin{equation}
\begin{split}
    \mathbf c_{\mathcal M}
    &=
    \sum_{m\in\mathcal M}w_m\mathbf a_m,\\
    \mathbf q_{\mathcal M}^{\star}
    &=
    \operatorname{softmax}(\mathbf c_{\mathcal M})
    =
    \operatorname{Normalize}\!\left(
        \prod_{m\in\mathcal M}\mathbf p_m^{\,w_m}
    \right),
    \label{eq:three_barycenter_views}
\end{split}
\end{equation}
where the product and exponentiation are component-wise and
\(\operatorname{Normalize}(\mathbf g)=\mathbf g/(\ones^\top\mathbf g)\).
\end{corollary}

\begin{proof}
By Proposition~\ref{prop:centered_coordinates},
\(\log\mathbf p_m=\mathbf a_m+b_m\ones\) for some scalar $b_m$. Hence
\[
    \sum_{m\in\mathcal M}w_m\log\mathbf p_m
    =
    \sum_{m\in\mathcal M}w_m\mathbf a_m
    +
    \left(\sum_{m\in\mathcal M}w_mb_m\right)\ones .
\]
The final term is class independent and disappears under softmax.
Combining this fact with Eq.~\eqref{eq:weighted_poe_solution} proves all
three forms.
\end{proof}

Thus, ``entropy-conditioned'' specifies how the sample-wise weights are
chosen, while ``reverse-KL barycenter,'' ``weighted product of experts,''
and ``weighted mean in centered-logit coordinates'' are three equivalent
descriptions of the resulting consensus.

\paragraph{Relation to probability averaging.}
The KL direction is essential to this form. Reversing the arguments gives
the forward-KL problem
\begin{equation}
    \arg\min_{\mathbf q\in\simplex}
    \sum_{m\in\mathcal M}
    w_m\operatorname{KL}(\mathbf p_m\Vert\mathbf q)
    =
    \sum_{m\in\mathcal M}w_m\mathbf p_m,
    \label{eq:forward_kl_mixture}
\end{equation}
which is the arithmetic mixture of expert probabilities. By contrast,
Eq.~\eqref{eq:weighted_poe_solution} pools class-wise log evidence. This
derivation establishes the geometry of the construction; it does not by
itself assert that one fusion rule must be more accurate.

\subsection{Decision-Retention Conditions under Expert Conflict}
\label{app:decision_retention}

We now specialize the initial construction to the active pair
\(\mathcal M=\{s,v\}\). For a fixed sample $x_i$, abbreviate
\begin{equation}
\begin{gathered}
    \mathbf p_s=\mathbf p_{s,i},\qquad
    \mathbf p_v=\mathbf p_{v,i}^{(0)},\\
    w_s=w_{s,i}^{(0)},\qquad
    w_v=w_{v,i}^{(0)},\qquad
    \mathbf q^{(0)}=\mathbf q_{\{s,v\},i}^{(0)} .
    \label{eq:initial_abbreviations}
\end{gathered}
\end{equation}
The weights are strictly positive and satisfy \(w_s+w_v=1\). Assume
first that each expert has a unique top-1 decision,
\begin{equation}
    y_s=\argmax_k p_{s,k},
    \qquad
    y_v=\argmax_k p_{v,k}.
    \label{eq:expert_top1}
\end{equation}
We first consider the generic no-tie case in which each expert has a
unique top-1 decision. Ties are discussed at the end of this subsection.

\begin{lemma}[Pairwise consensus margin]
\label{lem:pairwise_consensus_margin}
For any two classes $j$ and $k$, the initial consensus satisfies
\begin{equation}
    \log\frac{q_j^{(0)}}{q_k^{(0)}}
    =
    w_s\log\frac{p_{s,j}}{p_{s,k}}
    +
    w_v\log\frac{p_{v,j}}{p_{v,k}} .
    \label{eq:consensus_pairwise_margin}
\end{equation}
\end{lemma}

\begin{proof}
Take the logarithm of the ratio between the $j$th and $k$th entries in
Eq.~\eqref{eq:weighted_poe_solution}. The shared normalization constant
cancels.
\end{proof}

\paragraph{Agreement case.}
If \(y_s=y_v=y\), then both terms on the right-hand side of
Eq.~\eqref{eq:consensus_pairwise_margin} are strictly positive for every
\(k\neq y\). Therefore,
\begin{equation}
    y_s=y_v=y
    \quad\Longrightarrow\quad
    \argmax_k q_k^{(0)}=y .
    \label{eq:agreement_retention}
\end{equation}
Hence, conflict is the only case in which allocation must mediate
opposing top-1 evidence.

\begin{proposition}[Exact source-decision retention condition]
\label{prop:exact_source_retention}
Suppose \(y_s\neq y_v\). For each \(k\neq y_s\), define the positive
source margin and the VLM opposition against the source decision as
\begin{equation}
\begin{split}
    m_s(k)
    &=
    \log\frac{p_{s,y_s}}{p_{s,k}}>0,\\
    b_{v\rightarrow s}(k)
    &=
    \left[
        \log\frac{p_{v,k}}{p_{v,y_s}}
    \right]_+ ,
    \qquad [a]_+=\max\{a,0\}.
    \label{eq:source_margin_and_vlm_opposition}
\end{split}
\end{equation}
The initial consensus uniquely retains the source decision if and only if
\begin{equation}
    w_s m_s(k)
    >
    w_v b_{v\rightarrow s}(k)
    \qquad\text{for every }k\neq y_s .
    \label{eq:exact_source_retention}
\end{equation}
\end{proposition}

\begin{proof}
By Lemma~\ref{lem:pairwise_consensus_margin},
\begin{equation}
\begin{split}
    \log\frac{q_{y_s}^{(0)}}{q_k^{(0)}}
    &=
    w_sm_s(k)
    -
    w_v\log\frac{p_{v,k}}{p_{v,y_s}} .
    \label{eq:source_pairwise_consensus_margin}
\end{split}
\end{equation}
If \(p_{v,k}\leq p_{v,y_s}\), the second term in
Eq.~\eqref{eq:source_pairwise_consensus_margin} is nonnegative, so the
consensus strictly favors \(y_s\) over $k$ because \(m_s(k)>0\). If
\(p_{v,k}>p_{v,y_s}\), positivity of the consensus margin is exactly
\(w_sm_s(k)>w_vb_{v\rightarrow s}(k)\). Combining the two cases, the
consensus favors \(y_s\) over every competing class if and only if
Eq.~\eqref{eq:exact_source_retention} holds. This is equivalent to
\(y_s\) being the unique consensus maximizer.
\end{proof}

Because \(w_s/w_v=r_s/r_v\), the same condition can be written directly
in terms of the entropy-conditioned concentration scores:
\begin{equation}
    r_s m_s(k)
    >
    r_v b_{v\rightarrow s}(k)
    \qquad\text{for every }k\neq y_s ,
    \label{eq:source_retention_concentration_form}
\end{equation}
where
\[
    \begin{aligned}
        r_s&=\max\{\log K-H(\mathbf p_s),\epsilon\},\\
        r_v&=\max\{\log K-H(\mathbf p_v),\epsilon\}.
    \end{aligned}
\]
Equivalently, define
\begin{equation}
    \rho_{v\rightarrow s}
    =
    \max_{k\neq y_s}
    \frac{b_{v\rightarrow s}(k)}{m_s(k)} .
    \label{eq:source_critical_ratio}
\end{equation}
Then the exact condition is
\begin{equation}
    \frac{w_s}{w_v}
    =
    \frac{r_s}{r_v}
    >
    \rho_{v\rightarrow s}.
    \label{eq:source_ratio_condition}
\end{equation}
The maximum over all competitors is necessary: in a multiclass problem,
the weighted product may select a third class that is the top-1 decision
of neither expert. Comparing only \(y_s\) with \(y_v\) therefore does not
establish source-decision retention.

\begin{corollary}[Compact sufficient source condition]
\label{cor:compact_source_condition}
Define
\begin{equation}
\begin{split}
    M_s
    &=
    \min_{k\neq y_s}m_s(k),\\
    B_{v\rightarrow s}
    &=
    \max_{k\neq y_s}b_{v\rightarrow s}(k).
    \label{eq:compact_source_quantities}
\end{split}
\end{equation}
Then
\begin{equation}
    w_sM_s>w_vB_{v\rightarrow s}
    \quad\Longrightarrow\quad
    \argmax_kq_k^{(0)}=y_s .
    \label{eq:compact_source_sufficient_condition}
\end{equation}
\end{corollary}

\begin{proof}
For every \(k\neq y_s\),
\[
    w_sm_s(k)
    \geq w_sM_s
    >
    w_vB_{v\rightarrow s}
    \geq w_vb_{v\rightarrow s}(k).
\]
Proposition~\ref{prop:exact_source_retention} completes the proof.
\end{proof}

The compact inequality is sufficient but not necessary because the
smallest source margin and largest VLM opposition may occur at different
classes. Equation~\eqref{eq:exact_source_retention} or, equivalently,
Eq.~\eqref{eq:source_ratio_condition}, gives the exact condition.

\begin{proposition}[Exact VLM-decision retention condition]
\label{prop:exact_vlm_retention}
Suppose \(y_s\neq y_v\). For every \(k\neq y_v\), define
\begin{equation}
\begin{split}
    m_v(k)
    &=
    \log\frac{p_{v,y_v}}{p_{v,k}}>0,\\
    b_{s\rightarrow v}(k)
    &=
    \left[
        \log\frac{p_{s,k}}{p_{s,y_v}}
    \right]_+ .
    \label{eq:vlm_margin_and_source_opposition}
\end{split}
\end{equation}
The initial consensus uniquely retains the VLM decision if and only if
\begin{equation}
    w_vm_v(k)
    >
    w_sb_{s\rightarrow v}(k)
    \qquad\text{for every }k\neq y_v .
    \label{eq:exact_vlm_retention}
\end{equation}
Equivalently, with
\begin{equation}
    \rho_{s\rightarrow v}
    =
    \max_{k\neq y_v}
    \frac{b_{s\rightarrow v}(k)}{m_v(k)},
    \label{eq:vlm_critical_ratio}
\end{equation}
the condition is \(w_v/w_s=r_v/r_s>\rho_{s\rightarrow v}\).
\end{proposition}

\begin{proof}
Exchange the roles of $s$ and $v$ in the proof of
Proposition~\ref{prop:exact_source_retention}.
\end{proof}

A compact sufficient VLM condition follows symmetrically. Let
\begin{equation}
    M_v=\min_{k\neq y_v}m_v(k),
    \qquad
    B_{s\rightarrow v}
    =
    \max_{k\neq y_v}b_{s\rightarrow v}(k).
    \label{eq:compact_vlm_quantities}
\end{equation}
Then
\begin{equation}
    w_vM_v>w_sB_{s\rightarrow v}
    \quad\Longrightarrow\quad
    \argmax_kq_k^{(0)}=y_v .
    \label{eq:compact_vlm_sufficient_condition}
\end{equation}

\paragraph{Tie convention and scope of the result.}
Strict inequalities above guarantee a unique retained decision. Replacing
``$>$'' by ``$\geq$'' guarantees only that the corresponding expert
decision belongs to the consensus argmax set; the reported top-1 label
then depends on the deterministic tie-breaking rule. These conditions
characterize \emph{decision retention}, not correctness. A retained
source decision may be wrong, and a rejected source decision may be
correct. This is why the conflict-conditioned analysis in the main paper
uses labels only post hoc to separately measure source-only retention and
VLM-only absorption.

\subsection{Formal Properties of Consensus Shift Modulation}
\label{app:csm_properties}

We retain the notation of Eqs.~(5)--(6) in the main paper. At the
beginning of epoch \(e\in\{0,\ldots,E-1\}\), with \(E>1\), the
unmodulated dynamic consensus is evaluated over all \(N>1\) target
samples, producing \(\{\bar{\mathbf q}_i^{(e)}\}_{i=1}^N\). Define
\begin{equation}
\begin{split}
    h_i^{(e)}
    &=
    \frac{H(\bar{\mathbf q}_i^{(e)})}{\log K},\\
    u_i^{(e)}
    &=
    2\frac{\operatorname{rank}^{(e)}(h_i^{(e)})}{N-1}-1,\\
    d(e)
    &=
    1-\frac{e}{E-1},\\
    \gamma_i^{(e)}
    &=
    1+\lambda d(e)u_i^{(e)},
    \qquad 0\leq\lambda<1 .
    \label{eq:csm_factors}
\end{split}
\end{equation}
Here \(\operatorname{rank}^{(e)}(\cdot)\in[0,N-1]\) is the zero-based
average rank in ascending consensus entropy. For a step $\ell$ in epoch
$e$, CSM acts in centered-logit space as
\begin{equation}
\begin{split}
    \Delta\mathbf c_i^{(\ell)}
    &=
    \mathbf c_i^{(\ell)}-\mathbf c_i^{(0)},\\
    \widehat{\mathbf c}_i^{(\ell)}
    &=
    \mathbf c_i^{(0)}
    +
    \gamma_i^{(e)}\Delta\mathbf c_i^{(\ell)},\\
    \widehat{\mathbf q}_i^{(\ell)}
    &=
    \operatorname{softmax}
    \!\left(\widehat{\mathbf c}_i^{(\ell)}\right).
    \label{eq:csm_map}
\end{split}
\end{equation}

\begin{proposition}[Rank range, ordering, and monotone invariance]
\label{prop:csm_rank_properties}
For every target sample and epoch,
\begin{equation}
    -1\leq u_i^{(e)}\leq 1 .
    \label{eq:rank_factor_range}
\end{equation}
If \(h_i^{(e)}<h_j^{(e)}\), then
\(u_i^{(e)}\leq u_j^{(e)}\). Moreover, replacing all entropy values
within an epoch by any strictly increasing transformation leaves their
ranks, and hence all \(u_i^{(e)}\), unchanged.
\end{proposition}

\begin{proof}
The affine map \(r\mapsto 2r/(N-1)-1\) sends
\([0,N-1]\) to \([-1,1]\) and is increasing. A strictly increasing
transformation preserves all order relations and ties, including their
average ranks.
\end{proof}

The rank operation is therefore insensitive to the numerical spacing of
the consensus entropies across samples. This property applies only to the
cross-sample CSM rank in Eq.~\eqref{eq:csm_factors}; it does not alter the
raw within-sample expert allocation in
Eq.~\eqref{eq:entropy_conditioned_weights}.

\begin{proposition}[Bounded positive modulation]
\label{prop:csm_bounds}
For every \(i\) and \(e\),
\begin{equation}
    0
    <
    1-\lambda d(e)
    \leq
    \gamma_i^{(e)}
    \leq
    1+\lambda d(e)
    <
    2 .
    \label{eq:gamma_bounds}
\end{equation}
Consequently, CSM can neither cancel a nonzero dynamic shift nor reverse
its direction.
\end{proposition}

\par\smallskip
\begin{proof}
Equation~\eqref{eq:csm_factors} gives
\(d(e)\in[0,1]\), and
Proposition~\ref{prop:csm_rank_properties} gives
\(u_i^{(e)}\in[-1,1]\). Therefore,
\[
    1-\lambda d(e)
    \leq
    1+\lambda d(e)u_i^{(e)}
    \leq
    1+\lambda d(e).
\]
Because \(0\leq\lambda<1\), the lower endpoint is at least
\(1-\lambda>0\), and the upper endpoint is at most
\(1+\lambda<2\). A positive scalar cannot cancel or reverse a nonzero
vector.
\end{proof}

\par\smallskip
\begin{proposition}[Anchor-relative geometry]
\label{prop:csm_geometry}
For any norm \(\|\cdot\|\),
\begin{equation}
\begin{split}
    \widehat{\mathbf c}_i^{(\ell)}-\mathbf c_i^{(0)}
    &=
    \gamma_i^{(e)}
    \left(\mathbf c_i^{(\ell)}-\mathbf c_i^{(0)}\right),\\
    \left\|
        \widehat{\mathbf c}_i^{(\ell)}-\mathbf c_i^{(0)}
    \right\|
    &=
    \gamma_i^{(e)}
    \left\|
        \mathbf c_i^{(\ell)}-\mathbf c_i^{(0)}
    \right\|.
    \label{eq:csm_distance_scaling}
\end{split}
\end{equation}
If \(\Delta\mathbf c_i^{(\ell)}\neq\mathbf0\), then:
\begin{enumerate}
    \item \(0<\gamma_i^{(e)}<1\) places the modulated state strictly
    between the anchor and the dynamic state (\emph{contraction});
    \item \(\gamma_i^{(e)}=1\) leaves the dynamic state unchanged
    (\emph{identity});
    \item \(1<\gamma_i^{(e)}<2\) places it beyond the dynamic state on the
    same ray from the anchor (\emph{bounded extension}).
\end{enumerate}
In all three cases, the unmodulated and modulated shifts are positively
collinear. If \(\Delta\mathbf c_i^{(\ell)}=\mathbf0\), then
\(\widehat{\mathbf c}_i^{(\ell)}
=\mathbf c_i^{(\ell)}=\mathbf c_i^{(0)}\) for every admissible
\(\gamma_i^{(e)}\).
\end{proposition}

\par\smallskip
\begin{proof}
The first equality follows directly from Eq.~\eqref{eq:csm_map}; absolute
homogeneity of a norm and the positivity of $\gamma_i^{(e)}$ give the
second. The three cases follow from the value of the scalar.
\end{proof}

In particular, below-median entropy ranks
\(u_i^{(e)}<0\) contract the shift whenever \(\lambda d(e)>0\),
above-median ranks \(u_i^{(e)}>0\) extend it, and
\(u_i^{(e)}=0\) leaves it unchanged. This statement concerns relative
rank, rather than an absolute entropy threshold.

\par\smallskip
\begin{corollary}[Preservation of the centered-logit subspace]
\label{cor:csm_centered_subspace}
If \(\ones^\top\mathbf c_i^{(0)}=0\) and
\(\ones^\top\mathbf c_i^{(\ell)}=0\), then
\begin{equation}
    \ones^\top\widehat{\mathbf c}_i^{(\ell)}=0 .
    \label{eq:csm_centered_preservation}
\end{equation}
\end{corollary}

\par\smallskip
\begin{proof}
Equation~\eqref{eq:csm_map} is an affine combination of two vectors in
the linear subspace $\ones^\perp$.
\end{proof}

\par\smallskip
\begin{proposition}[Temporal decay to the identity]
\label{prop:csm_temporal_decay}
The deviation of the modulation factor from one satisfies
\begin{equation}
    \left|\gamma_i^{(e)}-1\right|
    =
    \lambda d(e)\left|u_i^{(e)}\right|
    \leq
    \lambda d(e).
    \label{eq:gamma_identity_deviation}
\end{equation}
Consequently,
\begin{equation}
\begin{split}
    \left\|
        \widehat{\mathbf c}_i^{(\ell)}
        -
        \mathbf c_i^{(\ell)}
    \right\|
    &=
    \left|\gamma_i^{(e)}-1\right|
    \left\|\Delta\mathbf c_i^{(\ell)}\right\|\\
    &\leq
    \lambda d(e)
    \left\|\Delta\mathbf c_i^{(\ell)}\right\|.
    \label{eq:csm_dynamic_deviation_bound}
\end{split}
\end{equation}
At the final epoch, \(d(E-1)=0\), so
\(\gamma_i^{(E-1)}=1\) and CSM is exactly the identity map. It is also
the identity for every epoch if \(\lambda=0\).
\end{proposition}

\par\smallskip
\begin{proof}
The first equality follows from Eq.~\eqref{eq:csm_factors}; the bound
uses \(|u_i^{(e)}|\leq1\). Subtracting
\(\mathbf c_i^{(\ell)}\) from both sides of
Eq.~\eqref{eq:csm_map} gives
\[
    \widehat{\mathbf c}_i^{(\ell)}-\mathbf c_i^{(\ell)}
    =
    \left(\gamma_i^{(e)}-1\right)
    \Delta\mathbf c_i^{(\ell)},
\]
and the norm bound follows. The endpoint statements are immediate.
\end{proof}

\paragraph{Effect on class log odds.}
Since both states are centered logits, CSM induces, for every pair of
classes \(j,k\),
\begin{equation}
\begin{split}
    \log
    \frac{\widehat q_{i,j}^{(\ell)}}
         {\widehat q_{i,k}^{(\ell)}}
    &=
    \widehat c_{i,j}^{(\ell)}
    -
    \widehat c_{i,k}^{(\ell)}\\
    &=
    (1-\gamma_i^{(e)})
    \left(c_{i,j}^{(0)}-c_{i,k}^{(0)}\right) \\
    &\quad+
    \gamma_i^{(e)}
    \left(c_{i,j}^{(\ell)}-c_{i,k}^{(\ell)}\right).
    \label{eq:csm_log_odds}
\end{split}
\end{equation}
For contraction, this is an interpolation of anchor and dynamic log
odds; for extension, it is an extrapolation along the same
anchor-relative direction.

\paragraph{Scope of the guarantees.}
The properties above establish boundedness, positive collinearity,
centered-coordinate preservation, and decay to the identity. They do not
guarantee that CSM preserves the top-1 class: the line in
Eq.~\eqref{eq:csm_map} may cross a decision hyperplane. Nor do they imply
that target accuracy improves monotonically. The empirical directional
validation in the main paper tests whether the proposed rank assignment
is beneficial for adaptation.

\section{Full Objective and Optimization Details}
\label{app:objective_optimization}

\subsection{Complete IIC Formulation}
\label{app:iic_formulation}

We first specify the batch-level IIC alignment
objective~\citep{ji2019invariant} used in Eq.~(7) of the main paper.
For a mini-batch
\(\mathcal B\) of size \(B=|\mathcal B|\), let
\(\mathbf p_{m,i}^{(\ell)}\in\Delta^{K-1}\) be the pre-update
prediction of branch \(m\in\{t,v\}\) for sample \(x_i\). The modulated
consensus is detached before either branch is optimized:
\begin{equation}
    \widetilde{\mathbf q}_i^{(\ell)}
    =
    \operatorname{sg}\!\left(
        \widehat{\mathbf q}_i^{(\ell)}
    \right),
    \qquad i\in\mathcal B,
    \label{eq:detached_consensus}
\end{equation}
where \(\operatorname{sg}(\cdot)\) denotes stop-gradient. Thus, the
consensus is fixed within the current update, while gradients are
retained through the prediction of the branch being optimized.

For branch \(m\), define the empirical soft joint distribution between
its predicted class \(a\) and consensus class \(b\) as
\begin{equation}
    J_{ab}^{(m)}
    =
    \frac{1}{B}
    \sum_{i\in\mathcal B}
    p_{m,i,a}^{(\ell)}
    \widetilde q_{i,b}^{(\ell)},
    \qquad a,b\in\{1,\ldots,K\}.
    \label{eq:iic_joint}
\end{equation}
Because both factors are categorical distributions,
\(\sum_{a,b}J_{ab}^{(m)}=1\). Its marginals are
\begin{equation}
    J_{a\cdot}^{(m)}
    =
    \sum_{b=1}^{K}J_{ab}^{(m)},
    \qquad
    J_{\cdot b}^{(m)}
    =
    \sum_{a=1}^{K}J_{ab}^{(m)}.
    \label{eq:iic_marginals}
\end{equation}
The IIC loss is the negative mutual information of this batch-level
joint distribution:
\begin{equation}
\begin{split}
    \mathcal L_{\mathrm{IIC}}
    \!\left(
        \mathbf p_m^{(\ell)},
        \widetilde{\mathbf q}^{(\ell)}
    \right)
    &=
    -\sum_{a=1}^{K}\sum_{b=1}^{K}
    J_{ab}^{(m)}
    \log
    \frac{J_{ab}^{(m)}}
    {J_{a\cdot}^{(m)}J_{\cdot b}^{(m)}}\\
    &=
    -I\!\left(P_m;\widetilde Q\right).
    \label{eq:iic_loss}
\end{split}
\end{equation}
This objective aligns class associations between each branch and the
shared consensus over the mini-batch. Unlike a sample-wise KL or soft
cross-entropy loss, it also contains the prediction marginals and
therefore favors informative, non-collapsed batch assignments.
In implementation, the joint distribution is numerically normalized
and its arguments are clipped by a small positive constant before
evaluating the logarithm.

\subsection{Complete \method{} Training Objective}
\label{app:complete_objective}

We next collect the shared supervision construction and all loss terms.
The initial consensus anchor is formed once from the frozen source
expert and the initially prompted VLM:
\begin{equation}
\begin{split}
    \mathbf c_i^{(0)}
    &=
    w_{s,i}^{(0)}\mathbf a_{s,i}
    +
    w_{v,i}^{(0)}\mathbf a_{v,i}^{(0)},\\
    \mathbf q_i^{(0)}
    &=
    \operatorname{softmax}\!\left(\mathbf c_i^{(0)}\right).
    \label{eq:initial_anchor_objective}
\end{split}
\end{equation}
The centered-logit anchor \(\mathbf c_i^{(0)}\) is cached and held fixed
during adaptation. At optimization step \(\ell\), the current target
and VLM branches form the unmodulated dynamic consensus
\begin{equation}
\begin{split}
    \mathbf c_i^{(\ell)}
    &=
    w_{t,i}^{(\ell)}\mathbf a_{t,i}^{(\ell)}
    +
    w_{v,i}^{(\ell)}\mathbf a_{v,i}^{(\ell)},\\
    \mathbf q_i^{(\ell)}
    &=
    \operatorname{softmax}\!\left(\mathbf c_i^{(\ell)}\right),
    \label{eq:dynamic_consensus_objective}
\end{split}
\end{equation}
where the weights are obtained from Eq.~(2) of the main paper with the
active expert set \(\mathcal M=\{t,v\}\). For a step in epoch \(e\),
CSM produces
\begin{equation}
\begin{split}
    \widehat{\mathbf c}_i^{(\ell)}
    &=
    \mathbf c_i^{(0)}
    +
    \gamma_i^{(e)}
    \left(
        \mathbf c_i^{(\ell)}-\mathbf c_i^{(0)}
    \right),\\
    \widehat{\mathbf q}_i^{(\ell)}
    &=
    \operatorname{softmax}
    \!\left(\widehat{\mathbf c}_i^{(\ell)}\right),\\
    \widetilde{\mathbf q}_i^{(\ell)}
    &=
    \operatorname{sg}
    \!\left(\widehat{\mathbf q}_i^{(\ell)}\right),
    \label{eq:modulated_teacher_objective}
\end{split}
\end{equation}
with \(\gamma_i^{(e)}=1+\lambda d(e)u_i^{(e)}\) as defined in
Eq.~(5) of the main paper. Equations~\eqref{eq:dynamic_consensus_objective}
and \eqref{eq:modulated_teacher_objective} are evaluated once from the
two pre-update predictions. The resulting detached distribution is the
common supervision snapshot for both branch objectives.

For the target branch, the hard consensus label and batch-mean
prediction are
\begin{equation}
    \widehat y_i^{(\ell)}
    =
    \argmax_{k}\widetilde q_{i,k}^{(\ell)},
    \qquad
    \overline{\mathbf p}_t^{(\ell)}
    =
    \frac{1}{B}
    \sum_{i\in\mathcal B}\mathbf p_{t,i}^{(\ell)}.
    \label{eq:hard_label_and_batch_mean}
\end{equation}
The hard-label classification loss is
\begin{equation}
    \mathcal L_{\mathrm{CE}}
    =
    -\frac{1}{B}
    \sum_{i\in\mathcal B}
    \log p_{t,i,\widehat y_i^{(\ell)}}^{(\ell)}.
    \label{eq:hard_consensus_ce}
\end{equation}
Using \(H(\mathbf p)=-\sum_kp_k\log p_k\), the complete branch-specific
objectives are
\begin{equation}
\begin{split}
    \mathcal L_v
    &=
    \mathcal L_{\mathrm{IIC}}
    \!\left(
        \mathbf p_v^{(\ell)},
        \widetilde{\mathbf q}^{(\ell)}
    \right),\\
    \mathcal L_t
    &=
    \alpha\,
    \mathcal L_{\mathrm{IIC}}
    \!\left(
        \mathbf p_t^{(\ell)},
        \widetilde{\mathbf q}^{(\ell)}
    \right)
    +
    \beta\,\mathcal L_{\mathrm{CE}}
    -
    \delta\,H\!\left(\overline{\mathbf p}_t^{(\ell)}\right),\\
    \mathcal L_{\mathrm{COSMO}}
    &=
    \mathcal L_t+\mathcal L_v .
    \label{eq:complete_cosmo_objective}
\end{split}
\end{equation}
Here, \(\alpha,\beta,\delta\geq0\) weight soft consensus alignment,
hard-label consolidation, and target-side prediction diversity,
respectively. Since the objective is minimized, the last term maximizes
the entropy of the target branch's batch marginal and discourages
category collapse. The VLM branch uses only soft consensus alignment:
its prompt context is optimized, while its image encoder, text encoder,
and logit scale remain frozen. The target branch is the deployed model
and therefore additionally receives hard-label and diversity terms.
Dataset-specific coefficient values are reported in Appendix~D.4.

\subsection{Training Algorithm}
\label{app:training_algorithm}

Let \(\theta_t\) denote the parameters of the target branch and
\(\phi\) the learnable VLM prompt context. Before adaptation,
\(f_t(\cdot;\theta_t)\) is initialized from the source model \(f_s\).
The frozen source expert and initially prompted VLM are evaluated over
the target set to construct and cache
\(\{\mathbf c_i^{(0)}\}_{i=1}^{N}\).

At the beginning of every epoch \(e\), the current target and VLM
branches are evaluated over the complete target set without parameter
updates. Their unmodulated dynamic consensuses
\(\{\bar{\mathbf q}_i^{(e)}\}_{i=1}^{N}\) determine the normalized
entropies, tie-aware ranks \(u_i^{(e)}\), and modulation factors
\(\gamma_i^{(e)}\) in Eq.~(5) of the main paper. These sample-indexed
factors are then fixed within that epoch.

\begin{algorithm}[h]
\caption{Optimization of \method}
\label{alg:cosmo_optimization}
\begin{algorithmic}[1]
\REQUIRE Unlabeled target set \(\mathcal D_t\), source expert \(f_s\),
pretrained VLM \(f_v\), and total epochs \(E\).
\STATE Initialize target branch \(f_t\leftarrow f_s\) and prompt
context \(\phi\).
\STATE Freeze \(f_s\), the VLM encoders, and the VLM logit scale.
\STATE Construct and cache anchors
\(\{\mathbf c_i^{(0)}\}_{i=1}^{N}\) using \(\mathcal M=\{s,v\}\).
\FOR{\(e=0,\ldots,E-1\)}
    \STATE Evaluate the unmodulated dynamic consensus over
    \(\mathcal D_t\).
    \STATE Compute \(\{u_i^{(e)},\gamma_i^{(e)}\}_{i=1}^{N}\).
    \FOR{mini-batch \(\mathcal B\subset\mathcal D_t\)}
        \STATE Compute pre-update predictions
        \(\mathbf p_t^{(\ell)}\) and \(\mathbf p_v^{(\ell)}\).
        \STATE Construct \(\mathbf c_i^{(\ell)}\) using
        \(\mathcal M=\{t,v\}\).
        \STATE Apply CSM relative to \(\mathbf c_i^{(0)}\) and detach
        \(\widetilde{\mathbf q}_i^{(\ell)}\).
        \STATE Form \(\mathcal L_t\) and \(\mathcal L_v\) from this
        same snapshot.
        \STATE Update \(\theta_t\) and \(\phi\) with their respective
        objectives.
    \ENDFOR
\ENDFOR
\RETURN Adapted target branch \(f_t\).
\end{algorithmic}
\end{algorithm}

Crucially, neither branch is updated before the shared supervision
snapshot for the current mini-batch has been formed. Although the two parameter groups
may be handled by separate optimizers, both losses use the same
pre-update detached consensus; no branch is supervised by the other
branch's post-update prediction. During adaptation, the trainable
components are the target branch and VLM prompt context. The source
expert, pretrained VLM encoders, VLM logit scale, cached anchors, and
the consensus snapshot within each update remain fixed. After
adaptation, only \(f_t\) is retained, so \method{} introduces no
additional branch or consensus computation at inference.

\section{Datasets and Evaluation Protocols}
\label{app:datasets_protocols}

\subsection{Dataset Statistics and Domain Splits}
\label{app:dataset_statistics}

We evaluate on the same four closed-set benchmarks as in the main
paper. Their statistics under the evaluated splits are summarized in
Table~\ref{tab:dataset_statistics}.

\begin{table}[h]
\centering
\small
\setlength{\tabcolsep}{2.6pt}
\renewcommand{\arraystretch}{1.08}
\begin{adjustbox}{max width=\linewidth}
\begin{tabular}{@{}lrrrr@{}}
\toprule
Dataset & Domains & Classes & Images & Tasks \\
\midrule
Office-31     & 3 & 31  & 4,652             & 6  \\
Office-Home   & 4 & 65  & 15,588            & 12 \\
VisDA-C       & 2 & 12  & 152,397/55,388    & 1  \\
DomainNet-126 & 4 & 126 & \(\sim\)145,000   & 12 \\
\bottomrule
\end{tabular}
\end{adjustbox}
\caption{Dataset statistics. VisDA-C image counts are reported as
Synthetic/Real.}
\label{tab:dataset_statistics}
\end{table}

\paragraph{Office-31.}
Office-31~\citep{saenko2010adapting} contains 4,652 images from
31 shared categories in three domains: Amazon (A), DSLR (D), and
Webcam (W).

\paragraph{Office-Home.}
Office-Home~\citep{venkateswara2017deep} contains 15,588 images from
65 categories in four domains: Art (A), Clipart (C), Product (P), and
Real World (R).

\paragraph{VisDA-C.}
VisDA-C~\citep{peng2018visda} is a synthetic-to-real benchmark with
12 categories. We use its training split of 152,397 rendered images as
the source domain and its validation split of 55,388 real images as the
target domain, following the standard unsupervised adaptation protocol.

\paragraph{DomainNet-126.}
DomainNet-126~\citep{saito2019semi} is the standard 126-category subset
used by prior VLM-guided SFDA work. It contains approximately 145,000
images from four DomainNet styles: Clipart (C), Painting (P), Real (R),
and Sketch (S).

\subsection{Adaptation Tasks and Evaluation Metrics}
\label{app:tasks_metrics}

All experiments follow closed-set adaptation: the source and target
domains share the same \(K\) categories. For a directed transfer
\(S\!\rightarrow\!T\), the source model is trained with labeled data
from \(S\); during adaptation, only this source-trained model, the
pretrained VLM, and unlabeled samples from \(T\) are available. Source
images and target labels are not accessed.

For Office-31, we evaluate all six ordered transfers among
\(\{\mathrm{A},\mathrm{D},\mathrm{W}\}\). For Office-Home, we evaluate
all 12 ordered transfers among
\(\{\mathrm{A},\mathrm{C},\mathrm{P},\mathrm{R}\}\). VisDA-C contains
the single Synthetic-to-Real transfer
\(\mathrm{S}\!\rightarrow\!\mathrm{R}\). For DomainNet-126, we evaluate
all 12 ordered transfers among
\(\{\mathrm{C},\mathrm{P},\mathrm{R},\mathrm{S}\}\).

For Office-31, Office-Home, and DomainNet-126, performance on a transfer
task \(\tau\) is target-domain top-1 accuracy:
\begin{equation}
    \operatorname{Acc}_{\tau}
    =
    \frac{1}{N_{\tau}}
    \sum_{i=1}^{N_{\tau}}
    \mathbb I
    \!\left[
        \argmax_k p_{t,i,k}=y_i
    \right].
    \label{eq:top1_accuracy}
\end{equation}
For VisDA-C, we report mean per-class accuracy,
\begin{equation}
    \small
    \operatorname{mAcc}_{\mathrm{VisDA}}
    =
    \frac{1}{K}
    \sum_{k=1}^{K}
    \frac{
        \sum_i
        \mathbb I[y_i=k]\,
        \mathbb I[\argmax_j p_{t,i,j}=k]
    }{
        \sum_i\mathbb I[y_i=k]
    },
    \label{eq:visda_mean_class_accuracy}
\end{equation}
which assigns equal weight to each of its 12 categories. Dataset-level
scores for the multi-transfer benchmarks are macro-averages over
directed tasks:
\begin{equation}
    \operatorname{Score}_{\mathcal D}
    =
    \frac{1}{|\mathcal T_{\mathcal D}|}
    \sum_{\tau\in\mathcal T_{\mathcal D}}
    \operatorname{Acc}_{\tau}.
    \label{eq:dataset_macro_average}
\end{equation}
Thus, each transfer task contributes equally, independently of its
number of target samples.

\subsection{Run, Aggregation, and Reporting Protocols}
\label{app:reporting_protocol}

Unless explicitly stated otherwise, every \method{} transfer experiment
is repeated with \(R=3\) independent adaptation runs. Let
\(A_{\tau,r}\) denote the evaluation score for task \(\tau\) in run
\(r\). We first average the runs within each task and then macro-average
the task means:
\begin{equation}
\begin{split}
    \overline A_{\tau}
    &=
    \frac{1}{R}\sum_{r=1}^{R}A_{\tau,r},\\
    \overline A_{\mathcal D}
    &=
    \frac{1}{|\mathcal T_{\mathcal D}|}
    \sum_{\tau\in\mathcal T_{\mathcal D}}
    \overline A_{\tau}.
    \label{eq:run_and_task_aggregation}
\end{split}
\end{equation}
The main tables report \(\overline A_{\tau}\) and
\(\overline A_{\mathcal D}\) as percentages, rounded only after
aggregation. Run-wise values and standard deviations are provided in
Appendix~F.3.

Target labels are used only after adaptation for computing the reported
metrics. They are not used to construct the consensus, optimize either
branch, select checkpoints, or tune hyperparameters. Labels used in
the mechanism analyses are likewise restricted to post-hoc
stratification and scoring.

For the Office-Home mechanism analyses in the main paper, statistics
are computed separately for each of the 12 transfer tasks and then
macro-averaged. Entropy bins and quintiles are formed within each task
before aggregation, preventing large target domains from dominating
the curves. Reported macro AUROC values average task-level AUROCs, and
the 95\% confidence intervals in Fig.~3 are computed across the 12
task-level statistics.

\begin{table*}[t]
\centering
\small
\setlength{\tabcolsep}{3.4pt}
\renewcommand{\arraystretch}{1.08}
\begin{adjustbox}{max width=\linewidth}
\begin{tabular}{l c c c c c c c c c c c}
\toprule
Dataset
& Target
& CLIP
& Epochs
& Batch
& $\eta_0$
& $\eta_\rho$
& $\alpha$
& $\beta$
& $\delta$
& $\epsilon$
& $\lambda$ \\
\midrule
Office-31
& R-50 & B/16 or B/32 & 30 & 64
& $1{\times}10^{-2}$ & $1{\times}10^{-3}$
& 1.3 & 0.10 & 1.0 & $1{\times}10^{-5}$ & 0.5 \\
Office-Home
& R-50 & B/16 or B/32 & 30 & 64
& $5{\times}10^{-3}$ & $5{\times}10^{-4}$
& 1.3 & 0.40 & 1.0 & $1{\times}10^{-5}$ & 0.5 \\
VisDA-C
& R-101 & B/16 or B/32 & 15 & 64
& $5{\times}10^{-3}$ & $5{\times}10^{-4}$
& 1.0 & 0.05 & 0.1 & $1{\times}10^{-5}$ & 0.5 \\
DomainNet-126
& R-50 & B/16 or B/32 & 30 & 64
& $5{\times}10^{-3}$ & $5{\times}10^{-4}$
& 1.3 & 0.40 & 0.5 & $1{\times}10^{-3}$ & 0.5 \\
\bottomrule
\end{tabular}
\end{adjustbox}
\caption{Target-adaptation configuration. CSM decays linearly over the
full training budget shown in the Epochs column.}
\label{tab:adaptation_hyperparameters}
\end{table*}

\section{Implementation Details}
\label{app:implementation_details}

\subsection{Source-Model Training}

Following the source-training protocol of
ProDe~\citep{tang2025prode}, we independently train a source model on
each labeled source domain. ProDe in turn bases its source-model
training on the official AdaContrast implementation
\citep{chen2022contrastive}; we retain this protocol and state the
benchmark-specific settings below. For each transfer task, the target
branch is initialized from the corresponding source model, while a
second frozen copy is retained as the source expert. After source
training, source images and labels are discarded and are never
accessed during target adaptation.

For Office-31 and Office-Home, the source model uses an
ImageNet-1K-pretrained ResNet-50 feature extractor; VisDA-C uses
ResNet-101. In all three cases, the feature extractor is followed by a
512-dimensional linear bottleneck with BatchNorm and a
weight-normalized linear classifier. We do not apply ReLU or dropout
after the bottleneck. DomainNet-126 follows the official AdaContrast
configuration with ResNet-50, a 256-dimensional bottleneck followed by
BatchNorm, and a weight-normalized linear classifier
\citep{chen2022contrastive}.

Source supervision is cross-entropy with label smoothing:
\begin{equation}
 \widetilde{\mathbf y}
 =(1-\sigma)\mathbf y+\frac{\sigma}{K}\mathbf 1,
 \qquad \sigma=0.1,
\end{equation}
where $K$ is the number of classes. For every benchmark, we randomly
split each labeled source domain into 90\% for training and 10\% for
validation, and retain the checkpoint with the highest
source-validation accuracy. The corresponding source-training settings
are summarized in Table~\ref{tab:source_training}.

\begin{table}[t]
\centering
\small
\setlength{\tabcolsep}{3.8pt}
\renewcommand{\arraystretch}{1.08}
\begin{adjustbox}{max width=\linewidth}
\begin{tabular}{l c c c c}
\toprule
Dataset & Head LR & Epochs & Batch & Backbone \\
\midrule
Office-31   & $1{\times}10^{-2}$ & 100 & 64 & R-50  \\
Office-Home & $1{\times}10^{-2}$ &  50 & 64 & R-50  \\
VisDA-C     & $1{\times}10^{-3}$ &  10 & 64 & R-101 \\
DomainNet-126 & $2{\times}10^{-3}$ & 60 & 128 & R-50 \\
\bottomrule
\end{tabular}
\end{adjustbox}
\caption{Source-model training settings. ``Head LR'' denotes the
initial learning rate of the bottleneck and classifier.}
\label{tab:source_training}
\end{table}

For Office-31, Office-Home, and VisDA-C, we use SGD with Nesterov
momentum $0.9$, weight decay $10^{-3}$, and the polynomial
learning-rate schedule
\begin{equation}
 \eta_j=\eta_0(1+10j/J)^{-0.75},
\end{equation}
where $j$ and $J$ are the current and total optimization steps. The
feature extractor uses $0.1\eta_0$, whereas the bottleneck and
classifier use $\eta_0$. For DomainNet-126, the official AdaContrast
configuration uses SGD with Nesterov momentum $0.9$, weight decay
$10^{-4}$, and cosine learning-rate decay. Its initial learning rates
are $2\times10^{-4}$ for the ResNet-50 backbone and
$2\times10^{-3}$ for the bottleneck and classifier
\citep{chen2022contrastive}. Source training applies resizing to
$256\times256$, random $224\times224$ cropping, horizontal flipping,
and ImageNet normalization. Validation uses a $224\times224$ center
crop after resizing.

\subsection{VLM and Prompt Configuration}

We use the OpenAI CLIP implementation~\citep{radford2021learning} with
ViT-B/32 or ViT-B/16.
Results obtained with the two VLM backbones are reported in separate
comparison groups. The CLIP image encoder, text encoder, token
embeddings, projection layers, and learned logit scale remain frozen
throughout adaptation. Only a learnable textual context shared across
classes is optimized.

The prompt contains four context tokens initialized from the CLIP token
embeddings of ``a photo of a'', followed by the fixed class name. Thus,
the initial prompt is ``a photo of a [CLASS].'' We use one prompt per
class without prompt ensembling or learnable class tokens. Class names
are read from the fixed benchmark lists, with underscores replaced by
spaces; no synonym expansion or plural rewriting is applied. Within our
implementation, the initial VLM prediction and subsequent prompt
adaptation use the same class names and initial template.

The VLM branch uses its deterministic preprocessing: bicubic resizing
of the shorter side to 224, a $224\times224$ center crop, RGB
conversion, and CLIP normalization. The target branch instead uses the
ResNet preprocessing described below. Accordingly, the two branches
observe different backbone-specific views of the same target sample;
we do not use weak--strong view consistency. CLIP and prompt forward
passes use mixed precision, while the main target-model optimization is
performed in FP32. We keep CLIP's learned logit scale fixed and do not
introduce an additional prediction temperature. Expert reliability is
computed directly from the current predictive distributions without
temperature calibration or multiplicative logit normalization.

\subsection{Target-Adaptation Optimization}

\paragraph{Trainable parameters.}
Target adaptation updates the target feature extractor, bottleneck,
classifier, and textual prompt context. The source expert and all
pretrained CLIP parameters remain frozen. We use two independent SGD
optimizers for the target and prompt branches. Both use Nesterov
momentum $0.9$, weight decay $10^{-3}$, and the polynomial schedule
above. Let $\eta_0$ denote the benchmark-specific base learning rate.
For Office-31, Office-Home, and VisDA-C, the bottleneck uses
$\eta_0$, while the feature extractor, classifier, and prompt context
use $0.1\eta_0$. For DomainNet-126, all target-model components use
$\eta_0$ and the prompt context uses $0.1\eta_0$.

\paragraph{Data processing and BatchNorm.}
All adaptation experiments use a target batch size of 64. Each epoch
contains one optimization pass over the unlabeled target set, preceded by
a no-update full-set scan that refreshes the target-set entropy ranks.
Target training resizes an image
to $256\times256$, applies a random $224\times224$ crop and horizontal
flip, and normalizes it with ImageNet statistics. Evaluation replaces
the random crop with a center crop. The target model uses ordinary
training-mode BatchNorm; no per-sample BatchNorm, strong augmentation,
or dual-view objective is used.

\paragraph{Consensus update.}
Before adaptation, we evaluate the frozen source expert and initially
prompted VLM to cache the initial-consensus anchor for each target
sample. At the beginning of every epoch, we evaluate the current
unmodulated dynamic consensus over the complete target set using
deterministic views and refresh the target-set entropy ranks. The
initial scan supplies the ranks for the first epoch.

Within each mini-batch, the entropy-conditioned PoE consensus is
constructed once from the two branches' pre-update predictions and then
modulated by CSM. The resulting distribution is detached and shared by
the target and prompt objectives. Although the two branches have
separate optimizers, both are updated against this same pre-update
snapshot; an update to one branch therefore cannot alter the
supervision used by the other branch in the same iteration. The source
expert participates only in the initial anchor and is not inserted into
the dynamic expert pair.

For CSM, the default modulation strength is $\lambda=0.5$. Its
differential effect decays linearly from one at the start of adaptation
to zero at the end of the fixed training budget. The entropy ranks and
the modulation are computed without target labels. No learning-rate
warm-up or gradient clipping is used. We report the target model from
the final epoch without target-label-based early stopping or model
selection; the VLM branch is discarded at inference.

\subsection{Hyperparameters and Reproducibility Settings}

Table~\ref{tab:adaptation_hyperparameters} lists the complete
benchmark-level configuration. Here, $\eta_0$ is the target-model base
learning rate and $\eta_\rho$ is the prompt learning rate. The
coefficients $\alpha$, $\beta$, and $\delta$ weight target-side
consensus alignment, hard-label classification, and prediction
diversity, respectively; $\epsilon$ is the numerical stabilizer used in
the entropy-based weighting and probability computations. The CSM
coefficient $\lambda$ is fixed across all experiments.

The loss coefficients follow the benchmark-specific protocols on which
our implementation is based. The only principal CSM hyperparameter,
$\lambda$, is set to 0.5 for every benchmark. Once selected, each
benchmark-level configuration is fixed across all of its directed
transfer tasks, without task-specific tuning. The sensitivity analysis
in the main paper varies the signed modulation coefficient and decay
onset on Office-Home A$\rightarrow$C; it does not change the default
settings used for the main results. Target labels are used only for
final evaluation and post-hoc mechanism analysis, never for adaptation,
hyperparameter selection, early stopping, or checkpoint selection.

\method{} is implemented in Python 3.10 with PyTorch 2.3 and CUDA 12.1.
Each transfer task is trained on a single NVIDIA A100 GPU. Unless
otherwise stated, \method{} results are averaged over three independent
runs with seeds 2020, 2021, and 2022. Reported baseline results retain
the evaluation and repetition protocols of their source papers.

\section{Comparison Protocol}
\label{app:comparison_protocol}

\subsection{Method Categorization and Backbone Matching}

Table~\ref{tab:method_categories} summarizes the data available to each
comparison group. Source and zero-shot CLIP~\citep{radford2021learning}
are non-adaptation references. SHOT~\citep{liang2020shot},
NRC~\citep{yang2021nrc}, AaD~\citep{yang2022aad},
AdaCon~\citep{chen2022contrastive}, and TPDS~\citep{tang2024tpds} are
conventional non-VLM SFDA methods. DAPL~\citep{ge2023domain},
PADCLIP~\citep{lai2023padclip}, ADCLIP~\citep{singha2023adclip},
PDA~\citep{bai2024pda}, and DAMP~\citep{du2024damp} are VLM-guided UDA
methods that retain labeled source data during adaptation.
DIFO~\citep{tang2024difo}, ProDe~\citep{tang2025prode}, and
VSFOT~\citep{han2026vsfot} perform VLM-guided adaptation without source
data, as does \method{}. No compared method uses target labels during
adaptation.

\begin{table}[t]
\centering
\small
\renewcommand{\arraystretch}{1.08}
\setlength{\tabcolsep}{3.0pt}
\begin{adjustbox}{max width=\linewidth}
\begin{tabular}{@{}l c l@{}}
\toprule
Setting & Source & Methods \\
\midrule
Reference
& -- 
& Source, zero-shot CLIP \\

\shortstack[l]{Non-VLM\\SFDA}
& No
& \shortstack[l]{SHOT, NRC, AaD, AdaCon,\\TPDS} \\

\shortstack[l]{VLM-guided\\UDA}
& Yes
& \shortstack[l]{DAPL, PADCLIP, ADCLIP,\\PDA, DAMP} \\

\shortstack[l]{VLM-guided\\SFDA}
& No
& \shortstack[l]{DIFO, ProDe, VSFOT,\\\method{}} \\
\bottomrule
\end{tabular}
\end{adjustbox}
\caption{
Comparison groups and access to labeled source data during
adaptation. A dash denotes a non-adaptation reference.
}
\label{tab:method_categories}
\end{table}

The deployed target architecture of \method{} is ResNet-50 on
Office-31, Office-Home, and DomainNet-126, and ResNet-101 on VisDA-C.
VLM-guided results are grouped by the CLIP image-encoder backbone,
ViT-B/16 or ViT-B/32, and best/second-best markings are assigned only
within the corresponding VLM-backbone group. Thus, ``matched
backbone'' in the main result tables refers specifically to the VLM
backbone. Conventional SFDA methods do not use a VLM and are therefore
not part of this matching. Target-model architectures follow the
reported protocol of each baseline and are stated separately when
reproduction is required; we do not assume that every reported
baseline shares our exact source checkpoint. Source-available
VLM-guided UDA methods are included as informative references but do
not have the same data-access setting as \method{}.

\begin{table*}[t]
\centering
\small
\setlength{\tabcolsep}{4.2pt}
\renewcommand{\arraystretch}{1.10}

\begin{adjustbox}{max width=\linewidth}
\begin{tabular}{@{}L{4.8cm} c L{2.6cm} L{7.2cm}@{}}
\toprule
Result(s)
& VLM
& Provenance
& Evaluation protocol \\
\midrule

Source
& --
& Our evaluation
& Direct target-domain evaluation of our trained source
  checkpoints; no target adaptation. \\

Zero-shot CLIP
& B/32
& Our evaluation
& CLIP ViT-B/32 with the fixed initial prompt; no target
  adaptation. \\

\method{}
& B/16 or B/32
& Our evaluation
& Three independent adaptation runs with seeds 2020, 2021,
  and 2022. \\

\midrule

Conventional SFDA
(SHOT, NRC, AaD, AdaCon, TPDS)
& --
& Original papers
& Values and evaluation protocols follow the corresponding
  cited papers. \\

Source-available VLM-guided UDA
(DAPL, PADCLIP, ADCLIP, PDA, DAMP)
& B/16
& Original papers
& Values and evaluation protocols follow the corresponding
  cited papers. \\

Source-free VLM-guided SFDA
(DIFO, ProDe, VSFOT)
& B/32
& Original papers
& Values follow the corresponding cited papers, except for
  ProDe on DomainNet-126 below. \\

\midrule

ProDe on DomainNet-126
& B/32
& Our reproduction
& One run with seed 2020 using a ResNet-50 target backbone
  for target-backbone matching. \\

\bottomrule
\end{tabular}
\end{adjustbox}

\caption{
Provenance of the results used in the main-text and appendix
comparison tables. Source, zero-shot CLIP, \method{}, and the indicated
ProDe result are evaluated in our pipeline; the remaining values are
taken from the corresponding cited papers. A dash indicates that no
VLM is used.
}
\label{tab:result_provenance}
\end{table*}

\subsection{Sources and Reproduction of Baseline Results}

Table~\ref{tab:result_provenance} summarizes the provenance of the
results used in the main-text and appendix comparisons. The Source and
zero-shot CLIP~\citep{radford2021learning} reference rows are obtained
from our own evaluations rather than copied from prior work. The Source
row directly evaluates our trained source checkpoints on the
corresponding target domains, while zero-shot CLIP uses CLIP ViT-B/32
with the fixed initial prompt and no target adaptation. The reported
\method{} results are likewise obtained from our implementation and,
unless otherwise stated, are averaged over three independent runs.

The remaining baseline results are taken from the corresponding cited
papers and therefore retain their original model-selection, repetition,
and rounding conventions, except for the ResNet-50
ProDe~\citep{tang2025prode} result on DomainNet-126 reproduced below.

For DomainNet-126, the default ProDe~\citep{tang2025prode} result does
not use the same ResNet-50 target backbone as \method{}. We therefore
reproduce ProDe
with ResNet-50, a 256-dimensional bottleneck, and CLIP ViT-B/32. The
reproduction covers all twelve directed transfers, uses the same
independently trained source models as \method{}, and follows a
30-epoch adaptation schedule with batch size 64 and base learning rate
$5\times10^{-3}$. We fix the configuration across all transfers and do
not use target labels for tuning, early stopping, or checkpoint
selection. A single run with seed 2020 yields the 12-task macro average
of 83.9\% reported in the main table.

\section{Complete Benchmark Results and Extensions}
\label{app:complete_results}

This section provides the complete Office-31 results, class-wise
VisDA-C results, run-level variability on Office-Home, additional
source-free comparisons with CLIP ViT-B/16, and the extension to
multiple foundation models. Unless stated otherwise, the result
provenance and comparison conventions follow
Sec.~\ref{app:comparison_protocol}.

\begin{table}[t!]
\centering
\begingroup
\small
\renewcommand{\arraystretch}{1.08}
\setlength{\tabcolsep}{1.0mm}
\begin{adjustbox}{max width=\linewidth}
\begin{tabular}{@{}l | c | c | c | rrrr@{}}
\toprule
\multirow{2}{*}{Method}
& \multirow{2}{*}{Venue}
& \multirow{2}{*}{VLM}
& \multirow{2}{*}{SF}
& \multicolumn{4}{c}{Office-31} \\
\cmidrule(lr){5-8}
& & &
& $\to$A & $\to$D & $\to$W & Avg. \\
\midrule
Source
& -- & -- & --
& 60.6 & 89.0 & 85.9 & 78.5 \\
CLIP
& ICML'21 & B/32 & --
& 77.3 & 81.3 & 79.6 & 79.4 \\
\midrule
SHOT
& ICML'20 & -- & \cmark
& 74.4 & 96.9 & 94.7 & 88.6 \\
NRC
& NeurIPS'21 & -- & \cmark
& 75.2 & 98.0 & 94.9 & 89.4 \\
AaD
& NeurIPS'22 & -- & \cmark
& 75.8 & 98.2 & 95.6 & 89.9 \\
AdaCon
& CVPR'22 & -- & \cmark
& 75.7 & 80.3 & 87.2 & 81.0 \\
TPDS
& IJCV'24 & -- & \cmark
& 75.6 & 98.5 & 96.6 & 90.2 \\
\midrule
DAPL
& TNNLS'23 & \multirow{3}{*}{B/16} & \xmark
& 81.1 & 81.5 & 81.1 & 81.2 \\
PDA
& AAAI'24 & & \xmark
& \underline{83.0} & \underline{95.5}
& \underline{95.1} & \underline{91.2} \\
\rowcolor[gray]{0.92}
\textbf{\method{}}
& -- & & \cmark
& \textbf{83.8} & \textbf{98.4}
& \textbf{96.9} & \textbf{93.0} \\
\midrule
DIFO
& CVPR'24 & \multirow{4}{*}{B/32} & \cmark
& 83.1 & 98.0 & \underline{96.4} & 92.5 \\
ProDe
& ICLR'25 & & \cmark
& 82.8 & \textbf{98.3} & \textbf{96.7}
& \underline{92.6} \\
VSFOT
& CVPR'26 & & \cmark
& \underline{83.4} & \textbf{98.3} & 96.3
& \textbf{92.7} \\
\rowcolor[gray]{0.92}
\textbf{\method{}}
& -- & & \cmark
& \textbf{83.7} & \underline{98.2} & 96.2
& \textbf{92.7} \\
\bottomrule
\end{tabular}
\end{adjustbox}
\endgroup
\caption{Closed-set adaptation results (\%) on Office-31,
aggregated by target domain. Each target-domain score averages the two
directed transfers ending in that domain. ``VLM'' denotes the
image-encoder backbone of the vision-language model, and ``SF''
indicates source-free adaptation. Best and second-best results are
shown in \textbf{bold} and \underline{underlined}, respectively,
within each matched VLM-backbone group.}
\label{tab:supp_office31}
\end{table}

\subsection{Complete Office-31 Results}
\label{app:office31_results}

Table~\ref{tab:supp_office31} reports Office-31 results aggregated by
target domain, where each target-domain score averages the two directed
transfers ending in that domain. This compact presentation highlights
the more challenging transfers to Amazon while retaining the complete
six-transfer average. With CLIP ViT-B/32, \method{} obtains the best
$\to$A result and ties the strongest matched-backbone source-free
baseline in average accuracy at 92.7\%. With CLIP ViT-B/16, \method{}
further reaches 93.0\% and ranks first on all three target-domain
aggregates within the matched-backbone group.

\subsection{Class-Wise Results on VisDA-C}
\label{app:visda_classwise}

Table~\ref{tab:supp_visda_classwise} complements the mean per-class
accuracy in the main paper with results for all 12 VisDA-C categories.
The source-available VLM-guided UDA methods in the ViT-B/16 block report
class-wise results in their original evaluations
\citep{lai2023padclip,singha2023adclip,bai2024pda,du2024damp} and are
included to align the comparison set with the corresponding block of
Table~1 in the main paper. The DIFO and ProDe class-wise values are
reported in their original papers
\citep{tang2024difo,tang2025prode}. VSFOT is omitted because it reports
only the aggregate VisDA-C score \citep{han2026vsfot}. The remaining
conventional SFDA rows follow the provenance summarized in
Sec.~\ref{app:comparison_protocol}.

Within the ViT-B/32 source-free group, \method{} achieves the highest
mean per-class accuracy at 91.4\%. With ViT-B/16, it further reaches
92.0\%, the best matched-backbone mean, and obtains the leading result
on bicycle, car, knife, person, plant, and truck. The class-wise changes
between the two \method{} backbones are not uniform, indicating that the
aggregate gain does not arise from every category improving equally.

\begin{table*}[t]
\centering
\begingroup
\small
\renewcommand{\arraystretch}{1.08}
\setlength{\tabcolsep}{2.5pt}
\begin{adjustbox}{max width=\linewidth}
\begin{tabular}{@{}l | c | c | c | *{13}{c}@{}}
\toprule
\multirow{2}{*}{Method}
& \multirow{2}{*}{Venue}
& \multirow{2}{*}{VLM}
& \multirow{2}{*}{SF}
& \multicolumn{13}{c}{VisDA-C} \\
\cmidrule(lr){5-17}
& & &
& plane & bcycl & bus & car & horse & knife
& mcycl & person & plant & sktbrd & train & truck
& Mean \\
\midrule
SHOT
& ICML'20 & -- & \cmark
& 95.0 & 87.4 & 80.9 & 57.6 & 93.9 & 94.1
& 79.4 & 80.4 & 90.9 & 89.8 & 85.8 & 57.5
& 82.7 \\
NRC
& NeurIPS'21 & -- & \cmark
& 96.8 & 91.3 & 82.4 & 62.4 & 96.2 & 95.9
& 86.1 & 80.6 & 94.8 & 94.1 & 90.4 & 59.7
& 85.9 \\
AaD
& NeurIPS'22 & -- & \cmark
& 97.4 & 90.5 & 80.8 & 76.2 & 97.3 & 96.1
& 89.8 & 82.9 & 95.5 & 93.0 & 92.0 & 64.7
& 88.0 \\
AdaCon
& CVPR'22 & -- & \cmark
& 97.0 & 84.7 & 84.0 & 77.3 & 96.7 & 93.8
& 91.9 & 84.8 & 94.3 & 93.1 & 94.1 & 49.7
& 86.8 \\
TPDS
& IJCV'24 & -- & \cmark
& 97.6 & 91.5 & 89.7 & 83.4 & 97.5 & 96.3
& 92.2 & 82.4 & 96.0 & 94.1 & 90.9 & 40.4
& 87.6 \\
\midrule
PADCLIP
& ICCV'23 & \multirow{5}{*}{B/16} & \xmark
& 98.1 & \underline{93.8} & 87.1 & \underline{85.5}
& 98.0 & 96.0 & 94.4 & \underline{86.0}
& \underline{94.9} & 93.3 & 93.5 & \underline{70.2}
& \underline{90.9} \\
ADCLIP
& ICCVW'23 & & \xmark
& \textbf{99.6} & 92.8 & \textbf{94.0} & 78.6
& \underline{98.8} & 95.4 & \textbf{96.8} & 83.9
& 91.5 & 95.8 & \textbf{95.5} & 65.7
& 90.7 \\
PDA
& AAAI'24 & & \xmark
& \underline{99.2} & 91.1 & \underline{91.9} & 77.1
& 98.4 & 93.6 & \underline{95.1} & 84.9
& 87.2 & \textbf{97.3} & \underline{95.3} & 65.3
& 89.7 \\
DAMP
& CVPR'24 & & \xmark
& 98.7 & 92.8 & 91.7 & 80.1
& \textbf{98.9} & \underline{96.9} & 94.9 & 83.2
& 93.9 & 94.9 & 94.8 & \underline{70.2}
& \underline{90.9} \\
\rowcolor[gray]{0.92}
\textbf{\method{}}
& -- & & \cmark
& 98.6 & \textbf{94.3} & 89.2 & \textbf{85.7}
& 98.1 & \textbf{98.9} & 93.9 & \textbf{86.4}
& \textbf{95.8} & \underline{97.1} & 94.9 & \textbf{71.2}
& \textbf{92.0} \\
\midrule
DIFO
& CVPR'24 & \multirow{3}{*}{B/32} & \cmark
& 97.5 & 89.0 & \textbf{90.8} & \textbf{83.5}
& \underline{97.8} & 97.3
& \underline{93.2} & 83.5 & \textbf{95.2}
& \underline{96.8} & 93.7 & 65.9
& 90.3 \\
ProDe
& ICLR'25 & & \cmark
& \textbf{98.3} & \textbf{92.4} & 86.6 & 80.5
& \textbf{98.1} & \underline{98.0}
& 92.3 & \underline{84.3} & 94.7
& \textbf{97.0} & \underline{94.1} & \textbf{75.6}
& \underline{91.0} \\
\rowcolor[gray]{0.92}
\textbf{\method{}}
& -- & & \cmark
& \textbf{98.3} & \underline{91.7} & \underline{88.6}
& \underline{81.7} & 97.4 & \textbf{98.7}
& \textbf{94.3} & \textbf{85.6} & \underline{95.0}
& 96.2 & \textbf{94.3} & \underline{74.8}
& \textbf{91.4} \\
\bottomrule
\end{tabular}
\end{adjustbox}
\endgroup
\caption{Class-wise closed-set adaptation results (\%) on VisDA-C.
Category abbreviations follow prior work: bcycl, mcycl, and sktbrd
denote bicycle, motorcycle, and skateboard. ``VLM'' denotes the
image-encoder backbone, and ``SF'' indicates source-free adaptation.
Best and second-best values are shown in \textbf{bold} and
\underline{underlined}, respectively, within each matched VLM-backbone
group. Tied second-best values are all underlined.}
\label{tab:supp_visda_classwise}
\end{table*}

\subsection{Run-Level Results and Variability}
\label{app:run_level_results}

Table~\ref{tab:officehome_run_variability} reports the individual runs
of \method{} on all 12 Office-Home transfers using three independent
seeds. All runs use the same CLIP ViT-B/32 configuration,
hyperparameters, and evaluation protocol. The reported standard
deviations are sample standard deviations across the three runs.

\method{} obtains an Office-Home average of
$86.410\pm0.084\%$, with run-level averages ranging from 86.345\% to
86.505\%. Across individual transfers, the sample standard deviation
ranges from 0.040 to 0.370 percentage points, with an average of 0.199
points. Thus, no evaluated transfer has a standard deviation above 0.4
points under this configuration.

The observed seed variation is small relative to the 1.1-point
difference between \method{} and the reported VSFOT average in the
matched ViT-B/32 comparison. This comparison characterizes empirical
stability rather than formal statistical significance, because only
three \method{} runs are available and the competing method was not
evaluated through paired repeated trials in the same implementation.

\begin{table}[t]
\centering
\small
\renewcommand{\arraystretch}{1.06}
\setlength{\tabcolsep}{3.2pt}
\begin{adjustbox}{max width=\linewidth}
\begin{tabular}{@{}lrrrrr@{}}
\toprule
Transfer & Seed 2020 & Seed 2021 & Seed 2022 & Mean & Std. \\
\midrule
A$\to$C & 76.480 & 76.230 & 76.550 & 76.420 & 0.168 \\
A$\to$P & 92.630 & 92.660 & 92.540 & 92.610 & 0.062 \\
A$\to$R & 92.026 & 91.757 & 92.307 & 92.030 & 0.275 \\
C$\to$A & 85.283 & 85.943 & 85.904 & 85.710 & 0.370 \\
C$\to$P & 92.220 & 92.400 & 92.520 & 92.380 & 0.151 \\
C$\to$R & 91.977 & 91.676 & 91.957 & 91.870 & 0.168 \\
P$\to$A & 83.993 & 84.324 & 84.643 & 84.320 & 0.325 \\
P$\to$C & 76.034 & 76.103 & 76.103 & 76.080 & 0.040 \\
P$\to$R & 91.887 & 91.376 & 91.927 & 91.730 & 0.307 \\
R$\to$A & 85.543 & 85.493 & 85.704 & 85.580 & 0.110 \\
R$\to$C & 75.490 & 75.720 & 75.200 & 75.470 & 0.261 \\
R$\to$P & 92.580 & 92.880 & 92.700 & 92.720 & 0.151 \\
\midrule
\textbf{Avg.}
& \textbf{86.345}
& \textbf{86.380}
& \textbf{86.505}
& \textbf{86.410}
& \textbf{0.084} \\
\bottomrule
\end{tabular}
\end{adjustbox}
\caption{Run-level target accuracies (\%) of \method{} with CLIP
ViT-B/32 on Office-Home. Mean and sample standard deviation are
computed over three independent runs using the unrounded values.
``Avg.'' macro-averages the 12 transfers separately within each run.}
\label{tab:officehome_run_variability}
\end{table}

\subsection{Additional Source-Free Comparisons with ViT-B/16}
\label{app:v16_source_free}

The main paper compares VLM-guided methods within matched CLIP-backbone
groups. Table~\ref{tab:supp_v16_source_free} further isolates
source-free methods evaluated with CLIP ViT-B/16. The DIFO and ProDe
values are taken from the ViT-B/16 evaluations reported by
ProDe~\citep{tang2025prode}; this DIFO configuration was not reported in
the original DIFO paper. All rows use a ResNet-50 target model, except
ResNet-101 on VisDA-C. \method{} improves over the strongest prior
source-free result by 0.8, 1.5, and 0.3 points on Office-31,
Office-Home, and VisDA-C, respectively, while matching ProDe on
DomainNet-126.

\begin{table}[t]
\centering
\begingroup
\small
\renewcommand{\arraystretch}{1.08}
\setlength{\tabcolsep}{3.4pt}
\begin{adjustbox}{max width=\linewidth}
\begin{tabular}{@{}lccccc@{}}
\toprule
Method & Venue & O-31 & O-Home & VisDA-C & DN-126 \\
\midrule
DIFO
& CVPR'24 & \underline{92.2} & 85.5 & 91.0 & -- \\
ProDe
& ICLR'25 & \underline{92.2} & \underline{86.9}
& \underline{91.7} & \textbf{88.1} \\
\rowcolor[gray]{0.92}
\textbf{\method{}}
& -- & \textbf{93.0} & \textbf{88.4}
& \textbf{92.0} & \textbf{88.1} \\
\bottomrule
\end{tabular}
\end{adjustbox}
\endgroup
\caption{Additional source-free VLM-guided comparisons using CLIP
ViT-B/16. O-31, O-Home, and DN-126 denote Office-31, Office-Home, and
DomainNet-126, respectively. Best and second-best values are shown in
\textbf{bold} and \underline{underlined}; a dash indicates that the
corresponding result was not reported.}
\label{tab:supp_v16_source_free}
\end{table}

\begin{table*}[t]
\centering
\begingroup
\small
\renewcommand{\arraystretch}{1.08}
\setlength{\tabcolsep}{4.0pt}
\begin{adjustbox}{max width=\linewidth}
\begin{tabular}{@{}l | *{13}{c}@{}}
\toprule
Method
& A$\to$C & A$\to$P & A$\to$R
& C$\to$A & C$\to$P & C$\to$R
& P$\to$A & P$\to$C & P$\to$R
& R$\to$A & R$\to$C & R$\to$P
& Avg. \\
\midrule
CoMA
& 83.8 & 95.6 & 94.2
& 89.1 & \underline{95.9} & 93.7
& 89.2 & 83.9 & 93.6
& 89.4 & 84.6 & \underline{95.8}
& 90.7 \\
\rowcolor[gray]{0.92}
\textbf{\method{} (frozen)}
& \underline{86.6} & \underline{95.8} & \underline{94.8}
& \textbf{90.2} & 95.7 & \underline{94.5}
& \underline{90.1} & \underline{86.6} & \underline{94.6}
& \underline{90.2} & \underline{86.3} & 95.7
& \underline{91.8} \\
\rowcolor[gray]{0.92}
\textbf{\method{} (text adapter)}
& \textbf{87.0} & \textbf{96.1} & \textbf{94.9}
& \underline{89.8} & \textbf{96.0} & \textbf{94.7}
& \textbf{90.2} & \textbf{87.0} & \textbf{94.8}
& \textbf{90.6} & \textbf{87.0} & \textbf{96.0}
& \textbf{92.0} \\
\bottomrule
\end{tabular}
\end{adjustbox}
\endgroup
\caption{Extension to multiple foundation models on Office-Home
(accuracy, \%). All methods use a ResNet-50 target model and combine
CLIP ViT-B/32 with InstructBLIP. ``Frozen'' fixes the caption-derived
classification proxy, while ``text adapter'' updates only the residual
adapter on the class-name embeddings. Best and second-best values are
shown in \textbf{bold} and \underline{underlined}, respectively.}
\label{tab:supp_multiple_fm}
\end{table*}

\subsection{Extension to Multiple Foundation Models}
\label{app:multiple_foundation_models}

We extend \method{} from two heterogeneous experts to a configuration
with a source-initialized ResNet-50 target branch, CLIP ViT-B/32, and
\texttt{Salesforce/\allowbreak instructblip-flan-t5-xl}~\citep{dai2023instructblip}. The
reliability scores are normalized jointly over the three active
predictions, after which consensus construction and CSM retain the same
forms as in the two-expert setting. This experiment evaluates whether
the shared-consensus formulation can incorporate an additional
foundation model whose output is not natively expressed as
classification logits.

InstructBLIP first generates a caption for each target image. We encode
the caption and class names with
\texttt{thenlper/\allowbreak gte-base}~\citep{li2023generaltextembeddings}; their cosine similarities,
scaled by 100, form the InstructBLIP classification proxy. In the
\emph{frozen} variant, the captions, GTE embeddings, and resulting
proxy logits remain fixed. In the \emph{text-adapter} variant, the
caption embeddings remain fixed, while a 256-dimensional residual
adapter is applied to the class-name embeddings and optimized with the
same consensus-alignment objective.

Table~\ref{tab:supp_multiple_fm} compares the two extensions with
CoMA~\citep{lee2025collaborativelearningmultiplefoundation}, which also
combines CLIP ViT-B/32 and InstructBLIP on Office-Home. The frozen
extension reaches 91.8\%, indicating that the additional foundation
model can be incorporated without updating its proxy. The lightweight
text adapter provides a further 0.2-point gain, reaching 92.0\% and the
best result on 11 of the 12 transfers; the frozen variant remains best
on C$\to$A.

\section{Additional Ablations}
\label{sec:additional_ablations}

This section complements the ablations in the main paper from three
perspectives. Section~\ref{sec:additional_topologies} evaluates CSM under
partial branch re-aggregation, Sec.~\ref{sec:additional_consensus_construction}
extends the initial-consensus comparison to all four benchmarks, and
Sec.~\ref{sec:objective_component_ablations} decomposes the complete training
objective. Unless otherwise specified, all experiments use CLIP ViT-B/32 and
follow the benchmark configurations and reporting protocols described in
Secs.~\ref{app:datasets_protocols} and~\ref{app:implementation_details}.

\subsection{CSM across Re-aggregation Topologies}
\label{sec:additional_topologies}

Table~\ref{tab:additional_topologies} extends the component ablation in the
main paper by applying CSM to the two partial re-aggregation variants. In
these variants, either the current target prediction or the current VLM
prediction is combined with the other branch's initial prediction; the joint
variant uses the current predictions of both branches.

CSM improves every re-aggregation topology. It raises target-only
re-aggregation by $0.5$/$0.6$ points on Office-Home/DomainNet-126 and
VLM-only re-aggregation by $0.2$/$0.4$ points. Applied to joint
re-aggregation, CSM further improves $85.8/84.6\%$ to $86.4/85.1\%$.
Thus, its benefit is not specific to the simultaneous evolution of both
branches.

Joint re-aggregation nevertheless remains beneficial. Full \method{}
outperforms the target-only variant with CSM by $0.5/1.1$ points and the
VLM-only variant with CSM by $1.1/1.8$ points. The results separate the two
effects: CSM improves the anchor-relative evolution induced by each topology,
while jointly re-aggregating both current branches provides the strongest
consensus supervision.

\begin{table}[t]
\centering
\small
\renewcommand{\arraystretch}{1.08}
\setlength{\tabcolsep}{2.0pt}
\begin{adjustbox}{max width=\linewidth}
\begin{tabular}{@{}lccc cc@{}}
\toprule
& \multicolumn{2}{c}{Re-aggregation}
& & \multicolumn{2}{c}{Target Acc.} \\
\cmidrule(lr){2-3}\cmidrule(l){5-6}
Setting
& Target & VLM & CSM
& \shortstack{Office-\\Home}
& \shortstack{DomainNet-\\126} \\
\midrule
VLM $\rightarrow$ Target
& -- & -- & -- & 81.2 & 80.3 \\
VLM $\leftrightarrow$ Target
& -- & -- & -- & 84.8 & 82.0 \\
\midrule
Fixed consensus
& \xmark & \xmark & \xmark & 82.6 & 80.8 \\
\multirow{2}{*}{Target re-agg.}
& \cmark & \xmark & \xmark & 85.4 & 83.4 \\
& \cmark & \xmark & \cmark & 85.9 & 84.0 \\
\multirow{2}{*}{VLM re-agg.}
& \xmark & \cmark & \xmark & 85.1 & 82.9 \\
& \xmark & \cmark & \cmark & 85.3 & 83.3 \\
Joint re-agg.
& \cmark & \cmark & \xmark & 85.8 & 84.6 \\
\midrule
\rowcolor[gray]{0.92}
\textbf{\method{}}
& \cmark & \cmark & \cmark
& \textbf{86.4} & \textbf{85.1} \\
\bottomrule
\end{tabular}
\end{adjustbox}
\caption{
Extended ablation of supervision topology, branch re-aggregation, and CSM
using CLIP ViT-B/32 (12-shift average target accuracies, \%). For the
shared-consensus variants, a checkmark indicates that the current prediction
of the corresponding branch enters re-aggregation; a cross indicates that
its initial prediction is used. Dashes denote directed-transfer baselines for
which re-aggregation and CSM are not applicable.
}
\label{tab:additional_topologies}
\end{table}

\begin{table*}[t]
\centering
\small
\renewcommand{\arraystretch}{1.08}
\setlength{\tabcolsep}{8.0pt}
\begin{adjustbox}{max width=\linewidth}
\begin{tabular}{@{}llcccc@{}}
\toprule
Predictor / fusion & Weighting
& Office-31 & Office-Home & VisDA-C & DomainNet-126 \\
\midrule
Source expert & -- & 78.50 & 59.30 & 49.00 & 54.70 \\
VLM expert    & -- & 79.40 & 79.10 & 87.30 & 80.20 \\
\midrule
\multirow{5}{*}{PoE}
& Uniform               & 87.24 & 81.16 & 83.87 & 80.77 \\
& Entropy (scale-norm.) & 86.47 & 77.53 & 80.39 & 76.75 \\
& Entropy (rank-norm.)  & 86.64 & 79.73 & 81.93 & 78.97 \\
& Entropy (T-aligned)   & 86.64 & 78.08 & 79.66 & 77.28 \\
\rowcolor[gray]{0.92}
& \textbf{Entropy (raw)}
& \textbf{87.31} & \textbf{81.58}
& \textbf{84.54} & \textbf{81.61} \\
\bottomrule
\end{tabular}
\end{adjustbox}
\caption{
Initial predictor and consensus accuracies (\%) using CLIP ViT-B/32. All
entries are evaluated before target adaptation. Office-31, Office-Home, and
DomainNet-126 are averaged over 6, 12, and 12 directed transfers,
respectively; VisDA-C reports mean per-class accuracy. ``Scale-norm.'',
``rank-norm.'', and ``T-aligned'' denote global logit-scale normalization,
within-expert entropy-rank normalization, and label-free mean-entropy
alignment, respectively. Bold marks the best PoE result, and the gray row is
the allocation used by \method{}.
}
\label{tab:initial_consensus_benchmarks}
\end{table*}

\subsection{Initial Consensus Construction across Benchmarks}
\label{sec:additional_consensus_construction}

Table~\ref{tab:initial_consensus_benchmarks} extends the initial-consensus
analysis to all four benchmarks. Because all entries are evaluated before
target adaptation, this experiment isolates consensus construction from the
subsequent optimization trajectory. The PoE variants share the same
centered-logit pooling and differ only in the sample-wise allocation cue.

Besides uniform weighting and the raw expert entropy used by \method{}, we
evaluate three label-free controls for heterogeneous expert scales.
Scale normalization applies global logit-scale normalization before entropy
computation; rank normalization replaces each expert's entropy by its
within-expert target-set rank; and temperature alignment selects temperatures
that match the experts' target-set mean entropy. Target labels are used only
to evaluate the resulting initial predictions.

Raw expert entropy yields the best PoE consensus on every benchmark. Relative
to uniform PoE, it improves initial accuracy by $0.07$, $0.42$, $0.67$, and
$0.84$ points on Office-31, Office-Home, VisDA-C, and DomainNet-126,
respectively. It also exceeds the stronger individual expert on three
benchmarks. On VisDA-C, the raw-entropy PoE remains $2.76$ points below the
VLM expert, but it is still the strongest among the evaluated PoE variants.

Among the scale-controlled alternatives, rank normalization performs best,
but remains $0.67$--$2.64$ points below raw entropy. Global scale
normalization and mean-entropy alignment lead to larger reductions on most
benchmarks. These results do not treat raw entropy as calibrated correctness;
they show that, among the evaluated label-free cues, the experts' native
predictive concentration provides the most effective relative allocation for
the initial consensus.

\subsection{Training-Objective Ablations}
\label{sec:objective_component_ablations}

Table~\ref{tab:objective_component_ablations} decomposes the objective in
Eq.~(7) of the main paper into VLM-side IIC alignment (V-IIC), target-side IIC
alignment (T-IIC), hard-consensus cross-entropy, and target-side prediction
diversity. Consensus construction, branch re-aggregation, and CSM are kept
unchanged; only the indicated loss terms are removed.

All four terms contribute to the final result. Removing V-IIC and T-IIC
reduces the Office-Home average by $1.37$ and $2.09$ points, respectively,
indicating that alignment of both evolving branches with the shared consensus
is beneficial. Removing prediction diversity causes a smaller
$0.87$-point reduction.

Hard-consensus classification provides the largest individual contribution:
without it, accuracy drops by $6.75$ points. The soft-alignment-only variant
reaches $77.91\%$, whereas hard CE alone obtains $84.92\%$. Nevertheless, the
full objective exceeds the hard-CE-only variant by $1.49$ points. Hard
consensus decisions therefore provide the primary instance-level signal,
while branch-wise soft alignment and prediction diversity offer complementary
collaborative and distribution-level regularization.

\begin{table}[t]
\centering
\small
\renewcommand{\arraystretch}{1.08}
\setlength{\tabcolsep}{2.7pt}
\begin{adjustbox}{max width=\linewidth}
\begin{tabular}{@{}lccccr@{}}
\toprule
Setting & V-IIC & T-IIC & Hard CE & Div. & Acc. \\
\midrule
\rowcolor[gray]{0.92}
\textbf{Full objective}
& \cmark & \cmark & \cmark & \cmark
& \textbf{86.41} \\
w/o V-IIC
& \xmark & \cmark & \cmark & \cmark
& 85.04 \\
w/o T-IIC
& \cmark & \xmark & \cmark & \cmark
& 84.32 \\
w/o Hard CE
& \cmark & \cmark & \xmark & \cmark
& 79.66 \\
w/o Diversity
& \cmark & \cmark & \cmark & \xmark
& 85.54 \\
\midrule
Soft alignment only
& \cmark & \cmark & \xmark & \xmark
& 77.91 \\
Hard CE only
& \xmark & \xmark & \cmark & \xmark
& 84.92 \\
\bottomrule
\end{tabular}
\end{adjustbox}
\caption{
Ablation of the \method{} training objective on Office-Home using CLIP
ViT-B/32 (12-shift average target accuracy, \%). V-IIC and T-IIC denote
VLM-side and target-side IIC alignment with the shared consensus; Hard CE
uses the hard consensus decision, and Div. denotes target-side prediction
diversity.
}
\label{tab:objective_component_ablations}
\end{table}

\section{Extended Mechanism Analysis}
\label{sec:extended_mechanism}

This section provides two complementary diagnostics on Office-Home
A$\to$C. We first examine the discriminability of the adapted target
features and then evaluate sensitivity to the objective coefficients.
These experiments are intended for post-hoc analysis and do not use
target labels during adaptation or model selection.

\subsection{Target Feature Discriminability}
\label{sec:feature_discriminability}

We compare Source, SHOT, DIFO, ProDe, \method{} w/o CSM, and the full
\method{}. SHOT is adapted for 15 epochs, whereas DIFO, ProDe, and the
two \method{} variants use 30 epochs. The full \method{} panel uses the
canonical target-set entropy ranking adopted in the main experiments.
For each method, we extract the 512-dimensional target bottleneck features
immediately before the classifier, apply
$\ell_2$ normalization, and use the same deterministic evaluation
pipeline. The t-SNE visualization~\citep{vandermaaten2008visualizing} uses the same samples from ten fixed
classes for every method, whereas the accuracy and feature-space metric
reported in each panel are computed over all 65 classes and all 4,365
target samples. Each t-SNE projection is fitted independently and is
therefore used only to inspect within-panel class mixing and local
separation; its coordinates, scale, and global shape are not comparable
across methods. The panel values are single-run diagnostics with seed 2020 rather than
the three-run averages in the main comparison tables. The feature-analysis
checkpoints and the efficiency-profiling runs in
Table~\ref{tab:computational_efficiency} were produced in separate
experimental executions; their small single-run accuracy differences are
therefore not used for cross-table comparisons.

\begin{figure*}[t]
    \centering
    \includegraphics[width=\textwidth]
    {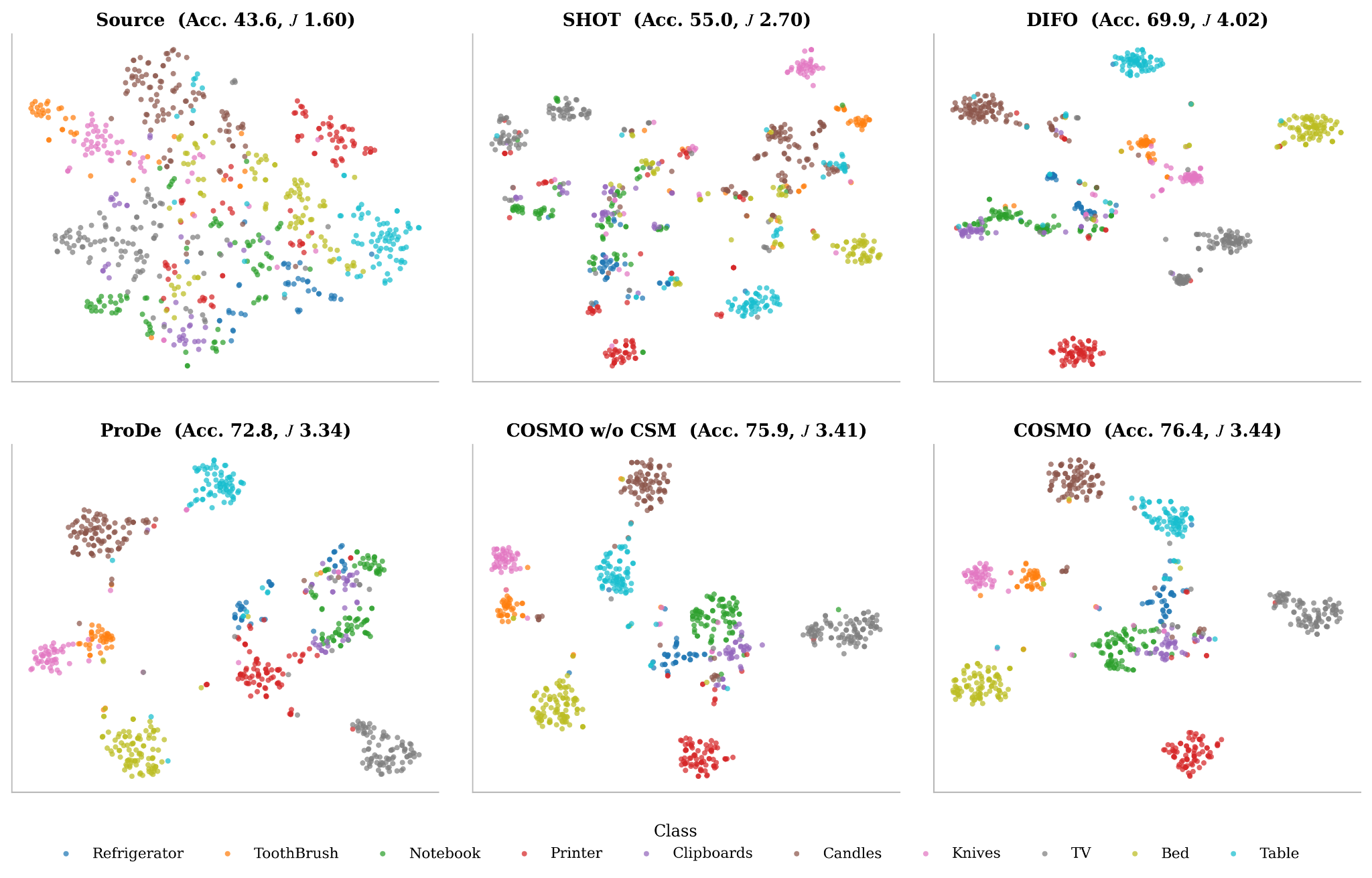}
    \caption{
    Target-feature visualization on Office-Home A$\to$C. Each panel
    shows an independently fitted t-SNE projection of the same target
    samples from ten fixed classes. Titles report target accuracy and
    the discriminability ratio $J$, both evaluated using all 65 classes.
    The ratio $J$ is computed in the original normalized
    512-dimensional feature space rather than in the t-SNE projection.
    Results use seed 2020; the full \method{} panel uses canonical
    target-set ranking. These diagnostic runs were conducted separately
    from the efficiency profiling in
    Table~\ref{tab:computational_efficiency}.
    }
    \label{fig:feature_discriminability}
\end{figure*}

Let $\boldsymbol{\mu}_k$ denote the normalized feature centroid of class
$k$. We measure within-class dispersion and between-class separation by
\begin{equation}
\begin{aligned}
D_{\mathrm{intra}}
&=
\frac{1}{K}\sum_{k=1}^{K}\frac{1}{n_k}
\sum_{i:y_i=k}
\left[1-\cos(\mathbf f_i,\boldsymbol{\mu}_k)\right],\\
D_{\mathrm{inter}}
&=
\frac{2}{K(K-1)}
\sum_{k<l}
\left[1-\cos(\boldsymbol{\mu}_k,\boldsymbol{\mu}_l)\right],
\\
J&=\frac{D_{\mathrm{inter}}}{D_{\mathrm{intra}}}.
\end{aligned}
\label{eq:feature_discriminability}
\end{equation}
A larger $J$ indicates greater inter-class separation relative to
within-class dispersion. This metric is used only to describe the
learned feature geometry and is not optimized during adaptation.

Figure~\ref{fig:feature_discriminability} shows that the unadapted
Source representation has substantial class overlap, with 43.6\%
accuracy and $J=1.60$. SHOT improves both quantities to 55.0\% and
2.70, but several selected classes remain locally mixed. DIFO produces
the largest $J$ of 4.02 and visibly compact clusters, yet reaches only
69.9\% accuracy; ProDe attains 72.8\% with $J=3.34$. The mismatch
between DIFO's high $J$ and its lower classification accuracy indicates
that compact class geometry alone does not determine the quality of the
deployed classifier. Well-separated feature clusters may still be
imperfectly aligned with the classifier decision regions.

The consensus-based variants achieve the highest target accuracies.
\method{} w/o CSM reaches 75.9\% with $J=3.41$, while the full method
further improves these values to 76.4\% and 3.44. Their projections
exhibit comparatively compact class-specific regions with less local
mixing than Source and SHOT. Compared with ProDe, \method{} does not
achieve the largest $J$, but it delivers higher target accuracy
together with slightly better class-discriminability statistics. The
smaller improvement from adding CSM suggests that it refines the
discriminative organization established by dynamic consensus learning
rather than fundamentally reorganizing the feature space. Overall, the
results support improved target-feature discriminability under
\method{}, while also showing that neither a two-dimensional projection
nor a single clustering statistic should be treated as a substitute for
target classification accuracy.

\subsection{Sensitivity to Objective Weights}
\label{sec:objective_sensitivity}

We further evaluate the coefficients in the training objective on
Office-Home A$\to$C. For this diagnostic, we write the total loss as
$\mathcal{L}_{t}+\lambda_v\mathcal{L}_{v}$, where the objective in
Eq.~(7) of the main paper corresponds to $\lambda_v=1$. Each of
$\lambda_v$, $\alpha$, $\beta$, and $\delta$ is varied independently
while all remaining coefficients are fixed at their defaults. We use
relative scales
$\{0.50,0.75,1.00,1.25,1.50\}$; the corresponding default coefficients
are $\lambda_v=1.00$, $\alpha=1.30$, $\beta=0.40$, and $\delta=1.00$.
Thus, for example, the sweep for $\alpha$ is
$\{0.65,0.975,1.30,1.625,1.95\}$. This analysis is not used for
task-specific hyperparameter selection.

\begin{figure}[t]
    \centering
    \includegraphics[width=\columnwidth]
    {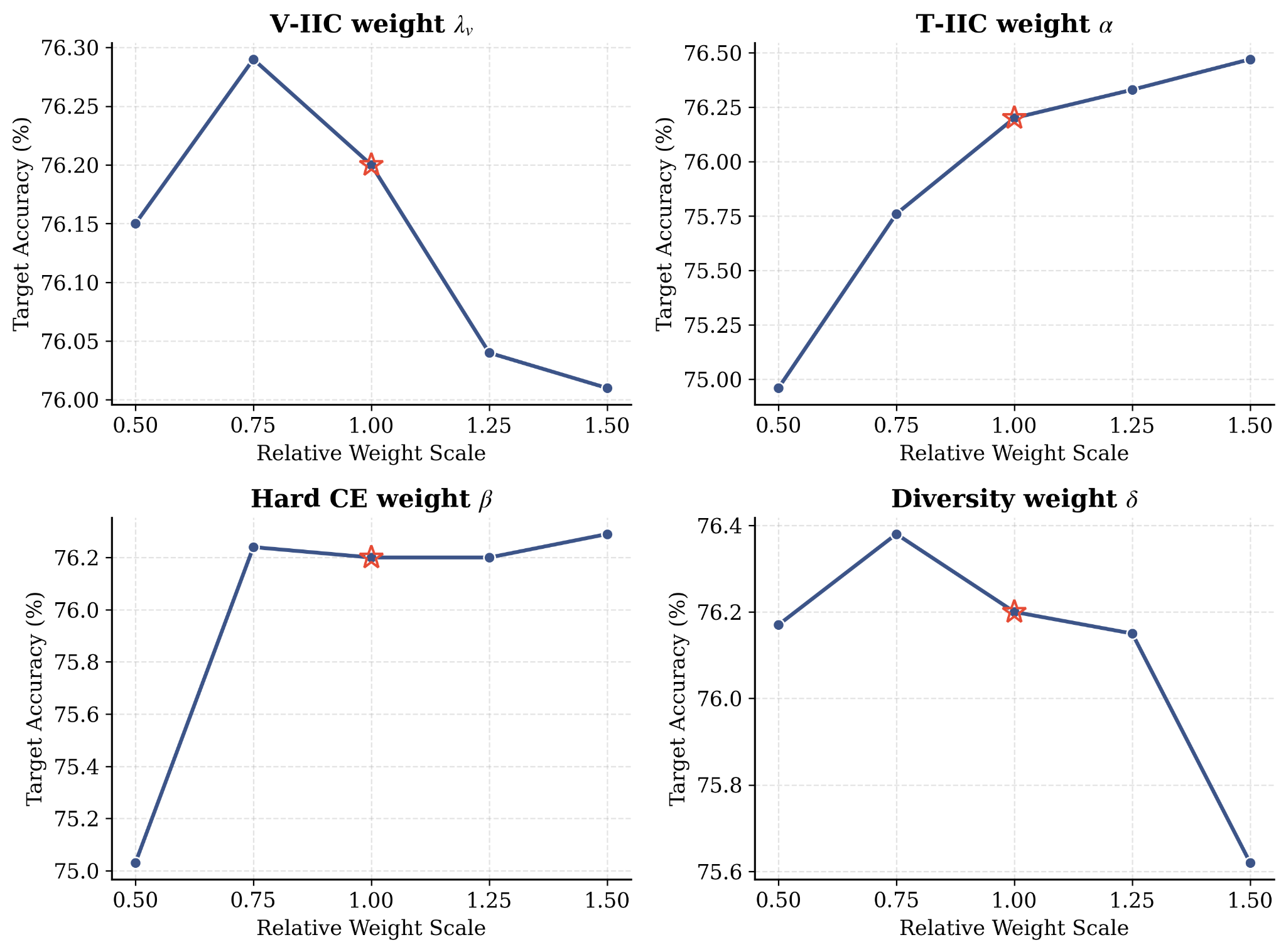}
    \caption{
    One-at-a-time sensitivity of target accuracy to the objective
    weights on Office-Home A$\to$C. The horizontal axis gives the scale
    applied to each coefficient's default value, and stars mark the
    default configuration. All other coefficients remain fixed.
    }
    \label{fig:objective_sensitivity}
\end{figure}

As shown in Fig.~\ref{fig:objective_sensitivity}, the VLM-side IIC
weight $\lambda_v$ is stable around its default: all settings remain
between 76.01\% and 76.29\%, and the default is only 0.09 points below
the best value. The hard-consensus CE weight $\beta$ also exhibits a
broad plateau from $0.75\times$ to $1.50\times$ its default, where the
accuracy varies only from 76.20\% to 76.29\%; a substantially smaller
weight of $0.50\times$ causes a clearer drop to 75.03\%. The diversity
weight $\delta$ remains near 76.2\% over scales from $0.50\times$ to
$1.25\times$, but excessive weighting at $1.50\times$ reduces accuracy
to 75.62\%.

The target-side IIC coefficient $\alpha$ shows the clearest directional
trend. Down-weighting it reduces accuracy to 74.96\% at $0.50\times$,
whereas larger values gradually improve performance to 76.47\% at
$1.50\times$. The default remains within 0.27 points of the best
A$\to$C setting and is kept fixed across all tasks and benchmarks rather
than tuned for this transfer. Taken together, the sweeps show that
\method{} is broadly insensitive to moderate changes in
$\lambda_v$, $\beta$, and $\delta$ around their defaults, while
sufficient target-side soft alignment is more important for this
particular transfer.

\begin{table*}[t]
\centering
\small
\renewcommand{\arraystretch}{1.10}
\setlength{\tabcolsep}{5.2pt}
\begin{adjustbox}{max width=\linewidth}
\begin{tabular}{@{}l c c c c c@{}}
\toprule
Method
& Target Acc. (\%)
& Total Time
& Time/Epoch (s)
& \shortstack{Effective Throughput\\(sample/s)}
& Peak Memory (GiB) \\
\midrule
SHOT
& 54.91 & 11.37 min & 22.74 & 191.9 & 23.60 \\
ProDe
& 72.85 & 23.96 min & 47.92 & 91.1 & 24.20 \\
\method{} w/o CSM
& 75.76 & 24.50 min & 48.99 & 89.1 & 24.21 \\
\rowcolor[gray]{0.92}
\method{} (target-set rank)
& 76.29 & 28.11 min & 56.21 & 77.7 & 24.21 \\
\method{} (batch-local rank)
& 76.43 & 24.07 min & 48.14 & 90.7 & 24.21 \\
\bottomrule
\end{tabular}
\end{adjustbox}
\caption{
Training efficiency on Office-Home A$\to$C. All entries use 30 epochs, target batch size 64, seed 2020,
and the same A100-PCIE-40GB GPU. ProDe and the \method{}
variants use CLIP ViT-B/32. The target-set-rank row is the canonical
\method{} configuration used in the main experiments; batch-local rank
is a practical approximation that computes entropy ranks within each
mini-batch. Effective throughput is the target-set size divided by the
wall-clock time per optimization epoch, including the rank-refresh scan when
applicable. Total time excludes one-time CLIP feature-cache construction.
}
\label{tab:computational_efficiency}
\end{table*}

\section{Computational Efficiency}
\label{sec:computational_efficiency}

We analyze the training cost of \method{} on Office-Home A$\to$C and
compare it with SHOT and ProDe. All methods are run for 30 epochs on a single NVIDIA
A100-PCIE-40GB GPU with target batch size 64 and seed 2020.
The VLM-based methods use CLIP ViT-B/32, and ProDe uses
a source-bank batch size of 64.
The reported time excludes the one-time construction of the CLIP
feature cache. Because these are single-run measurements, small
accuracy differences should not be interpreted as statistically
meaningful.

\paragraph{Computational complexity.}
For a mini-batch of size $B$ and $K$ classes, entropy-conditioned PoE
construction, consensus modulation, and the associated sample-wise
post-processing require $\mathcal{O}(BK)$ time and memory and introduce no
additional trainable module. The canonical target-set ranking additionally
performs one no-update forward pass of the current target and VLM branches
over all $N$ target samples per epoch. Once these predictions are available,
consensus construction, entropy computation, and sorting require
$\mathcal{O}(NK+N\log N)$ post-processing. The cached sample-wise logit
states require $\mathcal{O}(NK)$ storage; on A$\to$C, one float32 bank of
size $4{,}365\times65$ occupies only about 1.08 MiB.

Table~\ref{tab:computational_efficiency} shows that the core
consensus operations add negligible memory and little per-iteration
cost. Relative to ProDe, \method{} with batch-local ranking increases
time per epoch from 47.92 to 48.14 seconds, corresponding to only
0.46\% additional time, while improving target accuracy by 3.58
points. The peak-memory measurements are effectively identical after
rounding; the unrounded logs differ by approximately 7 MiB. The
w/o-CSM variant is similarly close to ProDe in both training time and
memory, indicating that PoE construction and sample-wise modulation are
not the main computational bottlenecks.

The additional cost of canonical \method{} mainly comes from the
epoch-level target-set scan used to refresh the global entropy ranks.
Replacing it with batch-local ranks reduces time per epoch from 56.21
to 48.14 seconds, a 14.36\% reduction, and increases effective throughput from
77.7 to 90.7 samples/s, a 16.73\% improvement. Its single-run accuracy
is 0.14 points higher than the canonical variant; this small difference
is not treated as evidence that batch-local ranking is more accurate.
The result instead shows that a local approximation can recover nearly
the same adaptation behavior when lower training cost is preferred.

CSM itself provides a positive accuracy contribution under both ranking
schemes: target-set ranking improves the w/o-CSM result from 75.76\% to
76.29\%, while batch-local ranking reaches 76.43\%. SHOT is faster
because it optimizes only the target branch, whereas ProDe and
\method{} jointly update a target model and a prompt branch. At
inference, \method{} discards the VLM branch, cached consensus states,
and CSM machinery, retaining only the adapted target model; it
therefore introduces no additional inference-time parameters or
computation.

\clearpage
\bibliographystyle{plainnat}
\bibliography{references}
\end{document}